\documentclass[a4paper,fleqn]{cas-dc}

\usepackage[authoryear,sort]{natbib}
\usepackage{graphicx}
\usepackage{comment}
\usepackage{amsmath}
\usepackage{colortbl}
\usepackage{siunitx}
\usepackage[table]{xcolor}
\usepackage{fancyhdr}
\usepackage{mathtools}
\usepackage{subcaption}
\AtBeginDocument{}
\ExplSyntaxOn
\cs_set:Npn \__make_fig_caption:nn #1#2
{
  \l_fig_align_tl
  \skip_vertical:N \l_fig_abovecap_skip
  \setbox\cascaptionbox=\hbox{%
     \sffamily\small\textbf{\color{scolor}#1.}~#2}
  \ifdim\the\wd\cascaptionbox<\dim_use:N \l_fig_width_dim\relax
        \parbox{ \l_fig_width_dim }
                {\unskip\ignorespaces\hfil\sffamily\small
       \textbf{\color{scolor}#1.}~#2\hfil\par }
  \else
        \parbox{ \l_fig_width_dim }
                {\rightskip=0pt\unskip\ignorespaces\sffamily
       \small\textbf{\color{scolor}#1.}~#2\par }
  \fi
  \skip_vertical:N \l_fig_belowcap_skip
}
\ExplSyntaxOff

\usepackage{amsthm, bm}
\usepackage{amssymb}
\usepackage{arydshln}
\usepackage{afterpage}
\usepackage{enumitem}
\usepackage{epstopdf}
\usepackage{orcidlink}
\usepackage{hyperref}

\newtheorem{theorem}{Theorem}

\newtheorem{remark}{Remark}
\newtheorem{definition}{Definition}

\newcommand{\rev}[1]{#1}
\newcommand{\revii}[1]{{\color{black} #1}}

\def\tsc#1{\csdef{#1}{\textsc{\lowercase{#1}}\xspace}}
\tsc{WGM}
\tsc{QE}
\tsc{EP}
\tsc{PMS}
\tsc{BEC}
\tsc{DE}
\ExplSyntaxOn
\bool_gset_true:N \g_stm_nologo_bool   % no e-mail/URL thumbnails (text label instead)
\ExplSyntaxOff
\RenewDocumentCommand \printorcid { }
  {
  }
\ExplSyntaxOn
\cs_new_protected:Npn \__arxiv_strip_orcid:n #1
  {
    \tl_set:Nn \l_tmpa_tl { #1 }
    \regex_replace_all:nnN { \c{orcidlink}\{.*?\} } { } \l_tmpa_tl
    \regex_replace_all:nnN { \c{textsuperscript}\{\} } { } \l_tmpa_tl
    \tl_use:N \l_tmpa_tl
  }
\RenewDocumentCommand \creditauthor { m m }
  {
    \seq_gput_right:Nn \g_stm_credits_seq
      { { \exp_not:N\bfseries \__arxiv_strip_orcid:n { #2 } : }\space #1 }
  }
\RenewDocumentCommand \emailauthor { m m }
  {
    \int_gincr:N \g_ead_int
    \tl_set:Nn \l_tmpa_tl { #2 }
    \regex_replace_all:nnN { \c{orcidlink}\{.*?\} } { } \l_tmpa_tl
    \regex_replace_all:nnN { \c{textsuperscript}\{\} } { } \l_tmpa_tl
    \seq_gput_right:Nx \g_stm_ead_seq
      {
        { \exp_not:n { \ttfamily \tl_to_str:n { #1 } } }
        \exp_not:n { \parsename } { \exp_not:V \l_tmpa_tl }
        \exp_not:n { \space(\eadauthor) }
      }
  }
\ExplSyntaxOff

\ExplSyntaxOn
\cs_set:Npn \__first_footerline: { \hbox{} }
\ExplSyntaxOff

\usepackage{etoolbox,xpatch}
\newcommand{\postprintnotice}{%
  \par\small\centering
  This is the accepted manuscript (post-print) of the article.
  The final published version (version of record) appeared in
  \textit{Control Engineering Practice}, 178 (2027), 107248,
  and is available at
  \url{https://doi.org/10.1016/j.conengprac.2026.107248}.
  \copyright{} 2026 Elsevier Ltd.
  This manuscript version is made available under the
  CC~BY-NC-ND~4.0 license:
  \url{https://creativecommons.org/licenses/by-nc-nd/4.0/}.
  \dashrule{0pt}{3pt}}
\xpretocmd{\Abstract}{\postprintnotice}{}{}

\begin{document} \let\WriteBookmarks\relax
    \def\floatpagepagefraction{1}
    \def\textpagefraction{.001}
    \shorttitle{Payload generalization on feedforward control actions for trajectory tracking in robots with rigid and elastic joints}
    \shortauthors{G. Pazzi et~al.}

    \title [mode = title]{Payload generalization on feedforward control actions for trajectory tracking in robots with rigid and elastic joints}
    \tnotemark[1]

    \tnotetext[1]{This research is partially supported by the European Union’s Horizon 2020 Research and Innovation Programme under Grant Agreement No. 101017274 (DARKO), in part by EU’s MSCA NEUTRAWEED grant no. 101182891, and in part by the Italian Ministry of Education and Research in the framework of the ``FoReLab'' (Future-oriented Research Lab) Project (Departments of Excellence). }

    \author[1,2]{Giorgia Pazzi\textsuperscript{\orcidlink{0009-0008-2606-0219}}}[%
        type=editor,
        auid=000,
        bioid=1,
        ]
    \cormark[1]
    \ead{giorgia.pazzi@phd.unipi.it}
    \credit{Writing – original draft, Validation, Software, Methodology, Formal analysis}

    \author[1,2]{Giorgio Simonini\textsuperscript{\orcidlink{0009-0006-4696-5889}}}[%
        type=editor,
        auid=000,
        bioid=2,
        ]
    \credit{Writing – original draft, Methodology}

    \author[1,2]{Michele Pierallini\textsuperscript{\orcidlink{0000-0003-0547-2747}}}[%
        type=editor,
        auid=000,
        bioid=3,
        ]
    \credit{Writing – review \& editing}

    \author[1,2]{Paolo Salaris\textsuperscript{\orcidlink{0000-0003-1313-5898}}}[%
        type=editor,
        auid=000,
        bioid=4,
        ]
    \credit{Writing – review \& editing, Funding acquisition}

    \author[1,2,3]{Giorgio Grioli\textsuperscript{\orcidlink{0000-0002-5310-2997}}}[%
        type=editor,
        auid=000,
        bioid=5,
        ]
    \credit{Writing – review \& editing, Conceptualization}

    \author[1,2]{Manolo Garabini\textsuperscript{\orcidlink{0000-0002-5873-3173}}}[%
        type=editor,
        auid=000,
        bioid=6,
        ]
    \credit{Writing – review \& editing, Supervision, Conceptualization}

    \author[1,2]{Franco Angelini\textsuperscript{\orcidlink{0000-0003-2559-9569}}}[%
        type=editor,
        auid=000,
        bioid=7,
        ]
    \credit{Writing – review \& editing, Supervision, Funding acquisition, Conceptualization}

    \affiliation[1]{organization={Centro di Ricerca ``En\-ri\-co Pi\-ag\-gio'', U\-ni\-ver\-si\-t\`{a} di Pisa},
                    addressline={Largo Lucio Lazzarino 1},
                    city={Pisa},
                    postcode={56126},
                    country={Italy}}

    \affiliation[2]{organization={Dipartimento di Ingegneria dell'Informazione, U\-ni\-ver\-si\-t\`{a} di Pisa},
                    addressline={Via G. Caruso 16},
                    city={Pisa},
                    postcode={56122},
                    country={Italy}}

    \affiliation[3]{organization={Soft Robotics for Human Cooperation and Rehabilitation, Fondazione Istituto Italiano di Tecnologia},
                    city={Genoa},
                    postcode={16163},
                    country={Italy}}

    \cortext[cor1]{Corresponding author.}

\begin{abstract}
Industries employ robots for repetitive tasks due to their high precision and performance, typically achieved through feedforward control actions based on an accurate robot description. However, an exact dynamic model can be difficult to obtain, especially for robots with elastic elements or when the payload changes.
Learning-based feedforward approaches are promising for trajectory tracking under model uncertainties, but typically face generalization issues when task parameters vary.
This paper proposes a novel method to generalize the feedforward control action with respect to payload parameters, independently of the underlying controller.
Firstly, the minimal number of example objects needed for payload generalization is derived.
Then, an algorithm is proposed to execute the same trajectory with a new payload, given only its parameters, without knowledge of the robot's dynamic model and even in the presence of elastic joints.
The effectiveness of the proposed algorithm is validated in simulations and experiments on tracking tasks with different robots and controllers.
Two sets of experiments are performed: on a standard Franka Emika Panda robot, and on a Franka Emika Panda robot equipped with an elastic wrist.
\end{abstract}


\maketitle

%%%%%%%%%%%%%%%%%%%%%%%%%%%%%%%%%%%%%%%%%%%%%%%%%%%%%%%%%%%%%%%%%%%%%%%%%%%
%% Introduction
%%%%%%%%%%%%%%%%%%%%%%%%%%%%%%%%%%%%%%%%%%%%%%%%%%%%%%%%%%%%%%%%%%%%%%%%%%%
\section{Introduction}
\label{sec:introduction}
Robotics proved to be a key technology in logistics and manufacturing~\citep{liu2024application}, demonstrating exceptional effectiveness in performing repetitive tasks such as pick-and-place operations~\citep{lobbezoo2021reinforcement, surati2021pick}, unwrapping pallets~\citep{gabellieri2022autonomous}, handling fragile items~\citep{Angelini_2020_softhandler, pagliarani2024variable}, and even throwing goods and objects~\citep{zeng2020tossingbot}.

The ability of robots to transform such industries comes from their safety during human-robot interaction tasks and their excellent tracking performance~\citep{palleschi2021fast}. 
Safety is increasingly achieved thanks to the physical adoption of elastic elements within the robot joints, leading to the so-called articulated soft robots~\citep{albu2008soft,della2020soft}. However, this progress comes at the expense of increased complexity in robot modeling and description. 
In scenarios where the robot's dynamic model is known with uncertainties, adaptive control techniques can be employed to estimate and compensate for unknown parameters.
\rev{Such schemes exploit the linearity of the manipulator dynamics with respect to the dynamic parameters, which they estimate online, from the known model structure, to compute the control torque~\citep{SlotineLi1987}; recent formulations even recover the exact parameters without requiring persistency of 
excitation~\citep{Arteaga2024}.} 
\begin{figure}%[t]
\centering
\includegraphics[width=\columnwidth]{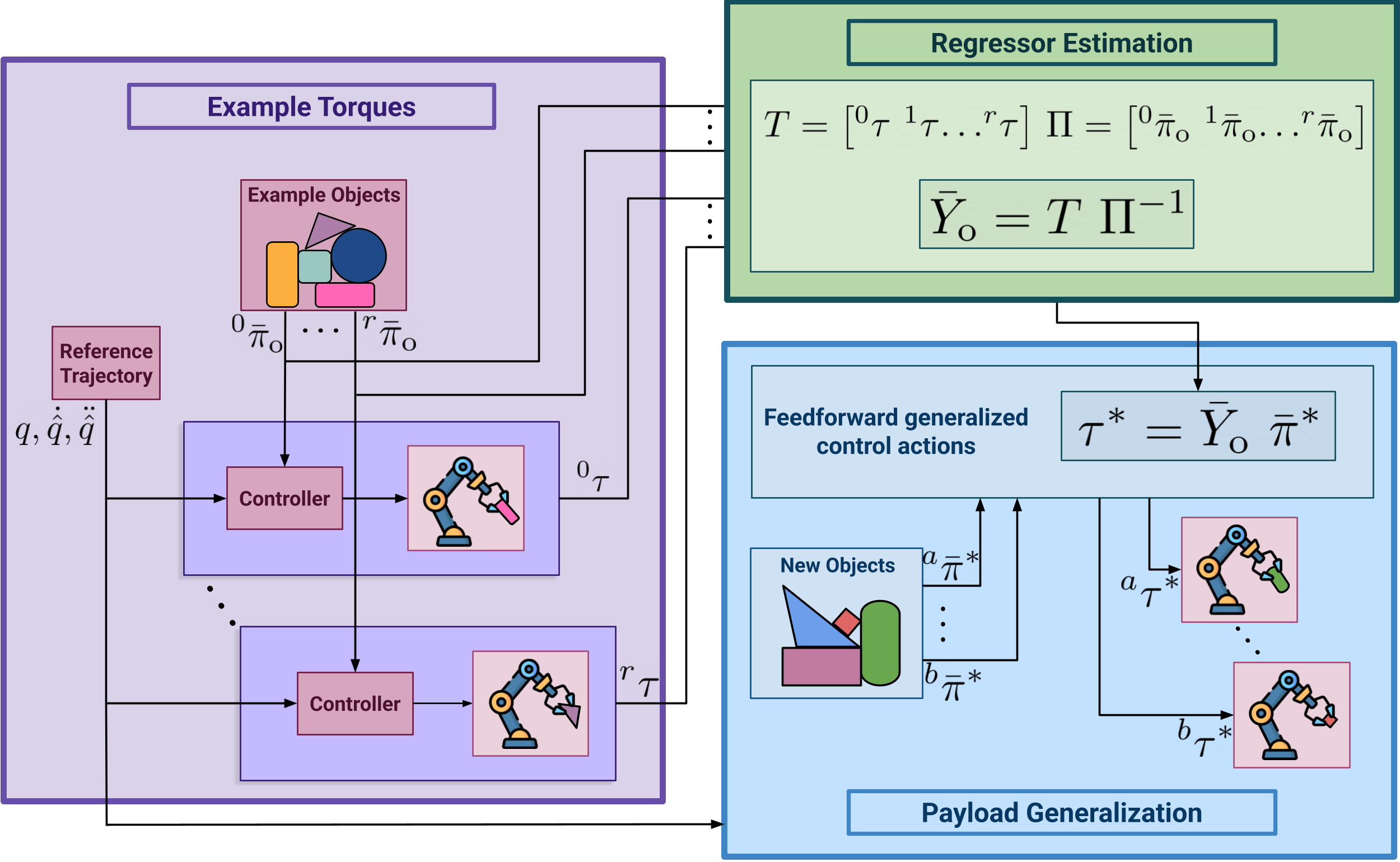}
    \caption{Scheme of the proposed payload generalization method. The figure shows the main steps of the procedure: first, feedforward torques are collected for a set of example objects with known parameters; then, these data are used to compute the regressor matrix that enables the generation of feedforward control actions for any new payload with known parameters.}
    \label{fig:genscheme}
\end{figure}

However, such feedback-based techniques often require high gains that stiffen dynamics, undermining the compliance of soft robots~\citep{pierallini2022swing}\rev{, and although recent adaptive computed-torque schemes mitigate this by enabling low feedback gains~\citep{adaptiveCT_simonini_2025}, they remain feedback- and model-based}. 
Additionally, adaptive control must re-initialize parameter estimates when the robot's dynamics change, leading to transient phases of suboptimal performance~\citep{middletone_adaptivect_1986}.
\rev{The same parameter linearity also underlies load identification, where the payload inertial parameters are estimated offline from measured joint torques and motion~\citep{Atkeson1986}: this is the inverse of the problem addressed here, which instead exploits known payload parameters to generate the corresponding feedforward torques.}

In contrast, good tracking performance is often achieved by feedforward control actions~\citep{liu2019industrial}, which are computed either from the perfect knowledge of the system's dynamics~\citep{chhatoi2023optimal} or through learning-based techniques~\citep{chen2024data}. 
Since the former approach relies on exact system modeling, it becomes particularly challenging in many real-world scenarios. For this reason, the latter model-free and learning-based control strategies have been increasingly investigated in recent years.

An example of this class of controllers is Iterative Learning Control (ILC)~\citep{Arimoto84, chen2020generalized, pierallini2025fishing}, which offers a powerful solution for improving the transient performance of systems that operate repetitively over a fixed time interval, finding application mainly in industrial robotics~\citep{Tayebi_2008, Saab_2022}. Unlike classical control techniques, ILC leverages previous iteration data to minimize the tracking error in a model-free, data-driven fashion~\citep{Bristow_2006}. 

Despite their promising performance, learning-based methods, particularly ILC, struggle to generalize pre-learned control actions to new scenarios.
Indeed, any slight change in the scene, such as the need to perform a different trajectory or move different payloads, leads to the use of a new control action and, since the model of the system is not entirely available, this is not trivial to achieve, and it requires the repetition of the entire learning phase whenever the task changes~\citep{li2022five, Angelini_2020_trajgen}.

Several studies have addressed the issue of generalization in learning-based controllers. For instance, the authors in~\citet{Mengacci_2021} proposed solutions to generalize the learned control action in terms of stiffness profile. In~\citet{Angelini_2020}, \citet{DeLuca2002} a procedure is developed to scale and generalize learned control actions to execute desired trajectories at varying velocities.
The problem is further examined from the perspective of final point~\citep{freeman2010iterative} and trajectory generalization in~\citet{Angelini_2020_trajgen}. 
However, \rev{none of these solutions generalizes the control action with respect to the payload's dynamic parameters, nor in a feedforward fashion for robots that may even present elastic joints}, which is a key feature in many pick-and-place tasks to execute precise, safe, and fast operations.
In many applications, such as manufacturing or logistics, robots repeatedly manipulate several objects, each with a different mass, inertia, or geometry. Conventional feedforward learning-based controllers typically require re-learning whenever some task parameters change, e.g., the manipulated object, leading to a time-consuming process that limits scalability.

This paper develops a novel method to generalize feedforward control actions, acquired with any controller, for tracking a desired trajectory with different payloads, Fig.~\ref{fig:genscheme}. 
To this end, it is first demonstrated that the number of examples required to generalize the feedforward control action corresponds to the number of dynamic parameters defining the object, plus one.
An algorithm is then introduced, capable of executing the same trajectory for a new payload based solely on its given parameters.
The core idea is to collect feedforward control actions for a small set of example objects. From this data, a regressor matrix is computed, enabling the estimation of the feedforward control required for any new object with known parameters. 
This makes the method easy to implement in industrial settings where such parameters are usually available or measurable.
Additionally, this purely data-driven, controller-independent approach requires no prior knowledge of the robot's dynamic model and can improve tracking performance in real-world applications, even when elastic elements are present in the robot's structure.

The effectiveness of the proposed method is validated for both rigid and soft articulated robots, showcasing its ability to operate independently of the manipulator's structure and controller while relying solely on the dynamic parameters used to describe the payload.
Specifically, experiments are also conducted on the Franka Emika Panda manipulator, equipped with a variable-stiffness actuator wrist and the SoftClaw\footnote{\href{https://qbrobotics.com/product/qb-softclaw/}{qb-softclaw}}.

This paper is organized as follows.
In Section~\ref{sec:problemstatement}, the problem, the system model, and the payload description are defined. 
The proposed payload generalization approach is detailed in Section~\ref{sec:problemsolution}, and Section~\ref{sec:validation} provides an extensive validation of the approach, including simulation results and real-world experiments. 
Finally, conclusions and future research directions are discussed in Section~\ref{sec:conclusion}.

%%%%%%%%%%%%%%%%%%%%%%%%%%%%%%%%%%%%%%%%%%%%%%%%%%%%%%%%%%%%%%%%%%%%%%%%%%%
%% Problem Statement
%%%%%%%%%%%%%%%%%%%%%%%%%%%%%%%%%%%%%%%%%%%%%%%%%%%%%%%%%%%%%%%%%%%%%%%%%%%
\section{Problem statement}
\label{sec:problemstatement}
This section defines the payload generalization problem for trajectory-tracking tasks for robots with rigid and elastic joints.

%%%%%%%%%%%%%%%%%%%%%%%%%%%%%%%%%%%%%%%%%%%%%%%%%%%%%%%%%%%%%%%%%%%%%%%%%%%%%
\subsection{System model}\label{subsec:system_model}
Let us consider a system with $n$ joints, of which $m$ are soft with $m \leq n$.
\revii{The model is derived under the standard reduced-model assumption of negligible inertia coupling between the link variables $q$ and the motor variables $\theta$, and of non-eccentric motor inertia~\citep{Spong_1987}.}
The dynamic equations of the system are inspired by 
\citet{Albu-Schäffer2016} and can be written as
{\small
\begin{align}
    &M(q)\ddot{q} + C(q, \dot{q})\dot{q} + G(q) - K(S_\theta \theta-q) = (I-S_\tau)\tau, \label{eq:linksdynamics} \\
    &J S_\theta\ddot{\theta} + DS_\theta\dot{\theta} + K(S_\theta\theta-q) = S_\tau\tau\;,\label{eq:motorsdynamics}
\end{align}
}
where the time dependence is omitted, $q \in \mathbb{R}^n$ includes the joint angles with their time derivatives, i.e., $\dot{q}$ and $\ddot{q}$, and $\theta \in \mathbb{R}^m$ indicates the motor angles, with their time derivatives, i.e., $\dot{\theta}$ and $\ddot{\theta}$. 
$I \in \mathbb{R}^{n \times n}$ is the identity matrix.
The matrices $M$, $C \in \mathbb{R}^{n \times n}$, and the vector $G \in \mathbb{R}^n$ represent the inertia matrix, the Coriolis and centrifugal matrix, and the gravitational force field, respectively. Recalling \eqref{eq:motorsdynamics}, $J$, $D$, and $K \in \mathbb{R}^{n \times n}$ of rank $m$ denote the motor inertia, friction, and stiffness matrices, respectively.
Since the motor dynamics matrices are of dimension $n \times n$, even though only $m$ joints are soft, the system is underactuated. To this end, the selection matrices, $S_{\theta} \in \mathbb{R}^{n \times m}$ and $S_{\tau} \in \mathbb{R}^{n \times n}$, are introduced, allowing the motor dynamics to be associated exclusively with the soft joints.
\revii{Let $\mathcal{E} = \{i_1, \dots, i_m\} \subseteq \{1, \dots, n\}$, with $i_1 < \dots < i_m$, collect the indices of the elastic joints, and let $e_i \in \mathbb{R}^n$ denote the $i$th column of $I$. 
The selection matrices are then
\begin{align}
    S_\theta &= \begin{bmatrix} e_{i_1} & \dots & e_{i_m} \end{bmatrix}\,, \label{eq:stheta} \\
    S_\tau &= \operatorname{diag}(s_1, \dots, s_n)\,,\label{eq:stau}
\end{align}
with $s_i = 1$ if $i \in \mathcal{E}$ and $s_i = 0$ otherwise.
Accordingly, $J$, $D$, and $K$ have null rows at the rigid-joint positions, hence their rank is $m$.}
Finally, $\tau \in \mathbb{R}^n$ is the actuation torque vector.
\revii{Under the reduced-model assumptions, a change of payload affects only the link dynamics~\eqref{eq:linksdynamics}: being rigidly attached to the end-effector, the payload modifies the inertial, Coriolis, and gravitational terms $M$, $C$, and $G$, which thus depend on its parameters.
Conversely, the motor dynamics~\eqref{eq:motorsdynamics} remains unaffected, since $J$, $D$, and $K$ characterize the actuators alone and the link--motor coupling enters only through the load-independent elastic term $K(S_\theta\theta-q)$.
Such dependence is left implicit here, and made explicit in Section~\ref{sec:problemsolution}.}

The following hypotheses are assumed for~\eqref{eq:linksdynamics}-\eqref{eq:motorsdynamics}:
\begin{itemize}
    \phantomsection\label{hp:sea}
    \item[H1)] The compliant joints are modeled as series elastic actuators (SEA) so that each soft joint is associated with a joint angle $q_i$ and motor angle $\theta_i \, \forall i$.
    \phantomsection\label{hp:diagandconstmat}
    \item[H2)] The matrices $J$, $D$, $K$ are diagonal and constant.  This reflects the mechanical decoupling and the linear elasticity of the SEA actuators, which are standard assumptions for the considered class of robots.
\end{itemize}

%%%%%%%%%%%%%%%%%%%%%%%%%%%%%%%%%%%%%%%%%%%%%%%%%%%%%%%%%%%%%%%%%%%%%%%%%%%%%
\vspace{-0.2em}
\subsection{Goal}
Consider a reference trajectory $\hat{q}: \left[0,\, t_\textnormal{f}\right] \rightarrow \mathbb{R}^n$ in the joint space.
Let us assume the nominal model of the robot \eqref{eq:linksdynamics}-\eqref{eq:motorsdynamics} is unknown, while the payload parameters are known and defined by the vector $\pi_\textnormal{o} \in \mathbb{R}^{r}$.
Assume also the torques required to track a desired trajectory are known for $p$ payloads (with parameters $^i\pi_\textnormal{o}$):
\begin{equation}\label{eq:exampletorques}
    ^i\tau(\cdot): \left[0,\, t_\textnormal{f}\right] \rightarrow \mathbb{R}^n, \quad \forall i = 0, \dots, p-1.
\end{equation}
The objectives of this work are:
\begin{itemize}
    \phantomsection\label{goal:g1}
    \item[G1)] to identify the minimal set of example objects, i.e., $p$, such that the corresponding feedforward torques 
    \({}^{i}\tau\) \(\forall i = 0, \dots, p-1\), allow for the computation of a well-defined control action;
    \phantomsection\label{goal:g2}
\item[G2)]  to develop a method for computing feedforward control actions that achieve a small trajectory tracking error $\hat{q}-q$, using the example torques \({}^{i}\tau,\) \(\text{for } i = 0, \dots, p-1\), under the system dynamics given by \eqref{eq:linksdynamics}-\eqref{eq:motorsdynamics}\rev{, and assuming $\hat{q}$ is feasible, i.e., there exist an initial state and an input $\tau(\cdot)$ such that the resulting motion satisfies $q(t)=\hat{q}(t)$, $\forall t \in \left[0,\, t_\textnormal{f}\right]$}.
\end{itemize}

%%%%%%%%%%%%%%%%%%%%%%%%%%%%%%%%%%%%%%%%%%%%%%%%%%%%%%%%%%%%%%%%%%%%%%%%%%%
%% Problem Solution
%%%%%%%%%%%%%%%%%%%%%%%%%%%%%%%%%%%%%%%%%%%%%%%%%%%%%%%%%%%%%%%%%%%%%%%%%%
\section{Proposed approach}
\label{sec:problemsolution}
This section presents the problem's solution.
As illustrated in Fig.~\ref{fig:genscheme}, the proposed method is organized into three main blocks:
\begin{enumerate}[label=\roman*)]
    \item Example Torques: in the first block, feedforward torques are collected for a set of example objects with known parameters. These data provide the basis for capturing the influence of payload parameters on the system dynamics.
    \item Regressor Estimation: in the second block, the collected data are exploited to compute the regressor matrix, which allows the system dynamics to be expressed in a linearly parameterized form.
    \item Payload Generalization: in the final block, the estimated regressor is used to compute the feedforward generalized torques for any new payload with known parameters.
\end{enumerate}

It is instrumental for the derivation of the proposed method to recall that the system model defined in~\eqref{eq:linksdynamics}-\eqref{eq:motorsdynamics}, under assumptions~\hyperref[hp:sea]{H1}--\hyperref[hp:diagandconstmat]{H2}, can be expressed as the product of a regressor matrix and the dynamic parameters vector of the system, see, e.g.,~\citet{Siciliano}. 
Specifically, \rev{given any payload,} \eqref{eq:linksdynamics}-\eqref{eq:motorsdynamics} can be written as 
\begin{equation}\label{eq:regressorform}
    Y(q,\dot{q},\ddot{q},\theta,\dot{\theta},\ddot{\theta}) \pi = \begin{bmatrix}   I-S_{\tau} \\ S_{\tau} \end{bmatrix} \tau \,,
\end{equation}
where $Y(q,\dot{q},\ddot{q},\theta,\dot{\theta},\ddot{\theta}) \in \mathbb{R}^{n \times l}$ is the regressor matrix, and $\pi \in \mathbb{R}^{l}$ is the system's dynamic parameters vector. Please note that~\citep{Siciliano} already presents the linear formulation of \eqref{eq:linksdynamics}, and recalling \hyperref[hp:diagandconstmat]{H2}, one can include the linear formulation of~\eqref{eq:motorsdynamics}.

The first step is to select the parameters of the object at the end-effector that influence the system's dynamics.
To do this, the parameters vector $\pi$ is decomposed into components that are directly associated with the different terms of \eqref{eq:linksdynamics}-\eqref{eq:motorsdynamics}.
This leads to rewriting \eqref{eq:regressorform} as follows
\begin{align}
    &Y_{\textnormal{s}} \pi_{\textnormal{s}} + Y_{\textnormal{o}} \pi_{\textnormal{o}} - Y_{q S_\theta \theta} \pi_K = (I - S_\tau) \tau \,,\label{eq:hybridreglink}\\
    &Y_{S_\theta \ddot{\theta}} \pi_J + Y_{S_\theta \dot{\theta}} \pi_D + Y_{qS_\theta \theta} \pi_K = S_\tau \tau\,, \label{eq:hybridregmotor}
\end{align}
where $\pi_\textnormal{s} \in \mathbb{R}^{l_\textnormal{s}}$ represents the dynamic parameters of the links, $\pi_\textnormal{o}\in \mathbb{R}^{r}$ corresponds to the parameters of the payload, and, recalling \hyperref[hp:diagandconstmat]{H2}, $Y_{S_\theta \ddot{\theta}}$, $Y_{S_\theta \dot{\theta}}$, and $Y_{qS_\theta \theta} \in \mathbb{R}^{n \times m}$  are matrices of rank $m$, and $\pi_J$, $\pi_D$, and $\pi_K\in\mathbb{R}^m$ are the inertia, damping and stiffness parameters vectors, respectively.
The links' dynamic parameters are denoted by the subscript $(\cdot)_\textnormal{s}$, and those of the payload are denoted with $(\cdot)_\textnormal{o}$.
\rev{Following this decomposition, a change of payload affects only the dynamic parameters collected in $\pi_\textnormal{o}$. The regressor matrix $Y$ depends only on the motion variables and not on the dynamic parameters, and the same holds for all the sub-blocks into which it is decomposed; they therefore remain unchanged when the payload varies~\citep{Atkeson1986, Hollerbach2016}. 
Note also that, as discussed in Section~\ref{sec:problemstatement}, the motor dynamics remains unaffected by the payload.}

\rev{It is instrumental for the method derivation to introduce the following definition.}

\begin{definition}\label{def:heterogeneous}
\rev{
A set of $p$ payloads, with dynamic parameters ${}^{i}\pi_{\textnormal{o}}$, $i=0,\dots,p-1$, is \emph{heterogeneous} if these vectors are affinely independent~\citep{Roman1992}. Equivalently, for any $i=0,\ldots,p-1$, the vectors ${}^{j}\pi_{\mathrm{o}}-{}^{i}\pi_{\mathrm{o}}$, $j\neq i$, are linearly independent.
}
\end{definition}

Theorem~\ref{thm:goal1} and Theorem~\ref{thm:goal2} are proposed to solve the objectives G1 and G2 defined in Section~\ref{sec:problemstatement}.

\hyperref[goal:g1]{G1} is addressed by the following Theorem.
\begin{theorem}\label{thm:goal1}

     Let us consider the system \eqref{eq:linksdynamics}-\eqref{eq:motorsdynamics}, under assumptions \hyperref[hp:sea]{H1}--\hyperref[hp:diagandconstmat]{H2}, and its linear formulation w.r.t. the dynamic parameters in \eqref{eq:hybridreglink}-\eqref{eq:hybridregmotor}.
     
    Let $r=\textnormal{dim}(\pi_\textnormal{o})$ be the number of dynamic parameters used to describe each payload.

    \rev{Let us consider a set of $p$ heterogeneous known payloads.}

    To compute generalized torques to execute the desired trajectory $\hat{q}(\cdot): \left[0,\, t_\textnormal{f}\right] \rightarrow \mathbb{R}^n$ with any new object, the number of distinct example objects must be at least
    \begin{equation}\label{eq:numberobjects}   
        p = r + 1 .
    \end{equation}
\end{theorem}
\begin{proof}\label{proof1}
Summing~\eqref{eq:hybridreglink} and~\eqref{eq:hybridregmotor} yields:
\begin{equation}\label{eq:hybridregextended}
    Y_{s} \pi_{s} + Y_{S_\theta \ddot{\theta}} \pi_J + Y_{S_\theta \dot{\theta}} \pi_D + Y_\textnormal{o} \pi_\textnormal{o} = \tau\,.
\end{equation}
In~\eqref{eq:hybridregextended}, one can combine all terms related to manipulator dynamics without payload at the end-effector as follows:
\begin{equation}
\label{eq:not_used}
    Y_\textnormal{s} \pi_\textnormal{s} + Y_{S_\theta \ddot{\theta}} \pi_J + Y_{S_\theta \dot{\theta}} \pi_D = Y_\textnormal{M}(q,\dot{q},\ddot{q},S_\theta\dot{\theta},S_\theta\ddot{\theta}) \pi_\textnormal{M}
\end{equation}
where $Y_\textnormal{M}(q,\dot{q},\ddot{q},S_\theta\dot{\theta},S_\theta\ddot{\theta}) = \begin{bmatrix}
Y_{s} & Y_{S_\theta \ddot{\theta}} & Y_{S_\theta \dot{\theta}} \end{bmatrix}$ and $\pi_\textnormal{M} = \begin{bmatrix} \pi_\textnormal{s}^\top & \pi_J^\top & \pi_D^\top \end{bmatrix}^\top$.
Thus,~\eqref{eq:hybridregextended} can be rearranged as
\begin{equation}\label{eq:dynregsep}
Y_\textnormal{M}(q,\dot{q},\ddot{q},S_\theta\dot{\theta},S_\theta\ddot{\theta}) \pi_\textnormal{M} + Y_\textnormal{o}(q,\dot{q},\ddot{q}) \pi_\textnormal{o} = \tau_\textnormal{M} + \tau_\textnormal{o} = \tau\,,
\end{equation}
where the torque vector $\tau \in \mathbb{R}^n$ is split into two components: $\tau_\textnormal{M} \in \mathbb{R}^n$ which is the nominal torque required to execute the desired motion $\hat q(t)$ without any object at the end-effector, and $\tau_\textnormal{o} \in \mathbb{R}^n$, which is the torque variation introduced by the presence of the object at the end-effector.
For the $i$th object, \eqref{eq:dynregsep} becomes:
\begin{equation}\label{eq:torqueobji}
    ^i\tau = \tau_\textnormal{M} + \ ^i\tau_\textnormal{o} = Y_\textnormal{M} \ \pi_\textnormal{M} + Y_\textnormal{o} \ ^i\pi_\textnormal{o}\,,
\end{equation}
where ${}^i\tau \in \mathbb{R}^n$ denotes the joint torques applied to the robot to follow the desired trajectory when the manipulator holds object $i$ at the end-effector.
Recalling~\eqref{eq:torqueobji}, the joint torques for object $i$ can be expressed as the product of the regressor matrix and dynamic parameter vector, i.e.,
\begin{equation}\label{eq:torqueregobji}
    ^i\tau(t) = \bar{Y}_\textnormal{o}(t) \ ^i\bar{\pi}_\textnormal{o},
\end{equation}
where $\bar{Y}_\textnormal{o}(t) = \begin{bmatrix} Y_\textnormal{M} \pi_\textnormal{M} & Y_\textnormal{o} \end{bmatrix}$, and $^i\bar{\pi}_\textnormal{o} = \begin{bmatrix} 1 & ^i\pi_\textnormal{o}^\top \end{bmatrix}^\top$.

Since payload parameters $^i\bar{\pi}_\textnormal{o},\forall i$ and joint torques for trajectory tracking $\, {}^i\tau(t), \forall i$ are known, what remains to be estimated is the regressor matrix, i.e., $\bar{Y}_\textnormal{o}(t)$.
To this end, the joint torques for a set of $p$ known objects are arranged into a matrix, namely $T(t) \in \mathbb{R}^{n \times p}$:
\begin{equation}\label{eq:matrixT}
    T(t) = \begin{bmatrix}^0\tau(t) & ^1\tau(t) & ^2\tau(t) & \ldots & ^{(p-1)}\tau(t)\end{bmatrix}\,.
\end{equation}
Grouping the known dynamic parameters $^i\bar{\pi}_\textnormal{o}$ for each object into a matrix $\Pi \in \mathbb{R}^{(r+1) \times p}$ yields 
\begin{equation}
    \Pi = \begin{bmatrix} ^0\bar{\pi}_\textnormal{o} & ^1\bar{\pi}_\textnormal{o} & ^2\bar{\pi}_\textnormal{o} & \ldots & ^{(p-1)}\bar{\pi}_\textnormal{o} \end{bmatrix}\,.
    \label{eq:matrixPi}
\end{equation}

\rev{Given~\eqref{eq:torqueregobji},~\eqref{eq:matrixT}, and~\eqref{eq:matrixPi}, the following equation can be written}
\begin{equation}\label{eq:matrixform}
    T(t) = \bar{Y}_\textnormal{o}(t) \Pi , \quad \rev{\forall t \in [0, t_\textnormal{f}]} .
\end{equation}

The dimensions of the matrices are $T(t) \in \mathbb{R}^{n \times p}$, $\Pi \in \mathbb{R}^{(r+1) \times p}$, and $\bar{Y}_\textnormal{o}(t) \in \mathbb{R}^{n \times (r+1)}$, where $r=\textnormal{dim}(\pi_\textnormal{o})$ is the number of payload dynamic parameters and $r+1=\textnormal{dim}(\bar{\pi}_\textnormal{o})$.
To generalize the control actions for different payloads, the regressor matrix $\bar{Y}_\textnormal{o}(t)$ is estimated from~\eqref{eq:matrixform}.
This requires matrix $\Pi \in \mathbb{R}^{(r+1) \times p}$ to have full row rank\rev{, which in turn implies $p \geq r+1$.
Choosing $p=r+1$ leads to a $\Pi$ matrix which is square and of full rank. Full rank is achieved by construction thanks to the heterogeneity of the example payloads. 
Indeed, subtracting the first column of $\Pi$ from all the others yields the square matrix
\begin{equation}\label{eq:matrixPitriang}
    \begin{bmatrix}
        1 & 0 & \ldots & 0 \\
        {}^{0}\pi_{\textnormal{o}} & {}^{1}\pi_{\textnormal{o}} - {}^{0}\pi_{\textnormal{o}} & \ldots & {}^{r}\pi_{\textnormal{o}} - {}^{0}\pi_{\textnormal{o}}
    \end{bmatrix}\,,
\end{equation}
which is block lower triangular with diagonal blocks of dimensions $1\times1$ and $r\times r$. 
The first block is the nonzero scalar 1, while the second one is nonsingular since, by Definition~\ref{def:heterogeneous}, the $r$ vectors ${}^{i}\pi_{\textnormal{o}} - {}^{0}\pi_{\textnormal{o}}$, $i = 1,\dots,r$, are linearly independent.
Hence $\Pi$ is invertible.
}

\rev{Together with the necessary condition $p \geq r+1$, this establishes} that $p=r+1$ is the minimum number of examples needed to generalize control actions across payloads.
\end{proof}

The result of Theorem~\ref{thm:goal1} establishes that the number of required example objects depends exclusively on the number of dynamic parameters describing the payload. 
Notably, the generalization method is independent of the manipulator’s structure, including the number of degrees of freedom. This independence represents a significant advantage as it eliminates the need for complex and often imprecise identification of the robot’s dynamic parameters. 

\begin{remark}\label{remark1}
    From an implementation perspective, it is noteworthy that \rev{the matrix $\Pi$, defined in~\eqref{eq:matrixPi}, is invertible whenever the example payloads are heterogeneous in the sense of Definition~\ref{def:heterogeneous}, i.e., whenever its columns ${}^i \bar{\pi}_{\textnormal{o}}$ are linearly independent. However, heterogeneity only guarantees invertibility in a qualitative sense. From a numerical standpoint, what matters is how well-conditioned $\Pi$ is, as quantified by its condition number $\kappa(\Pi) = \|\Pi\|\,\|\Pi^{-1}\|$. A small $\kappa(\Pi)$ requires 
    ${}^i \bar{\pi}_{\textnormal{o}}$ to be not merely linearly independent, but also far from linear dependence. Otherwise, $\Pi$ becomes ill-conditioned, i.e., $\|\Pi^{-1}\|$ grows large, compromising the} numerical stability and accuracy of the regressor matrix estimation, i.e., $\bar{Y}_{\textnormal{o}}$.
\end{remark}

\begin{remark}\label{remark_object0}
\rev{Without loss of generality, the case of no payload can be treated as one of the $p$ examples of heterogeneous payloads. Indeed, the absence of a payload corresponds to the dynamic parameter vector ${}^0\pi_\textnormal{o}=\mathbf{0} = [0 \ldots 0]^\top\in\mathbb{R}^r$. Given Definition~\ref{def:heterogeneous}, having heterogeneous objects implies having linearly independent vectors ${}^{j}\pi_{\textnormal{o}} - {}^{0}\pi_{\textnormal{o}} = {}^{j}\pi_{\textnormal{o}}, j \ne 0$. This means that the dynamic parameter vectors of the remaining $p-1$ objects are linearly independent. Substituting ${}^0\pi_\textnormal{o}$ into \eqref{eq:matrixPi} yields
\begin{equation}
    \Pi =
    \begin{bmatrix}
        1 & 1 & 1 & \ldots & 1\\
        \mathbf{0} & {}^1\pi_\textnormal{o} & {}^2\pi_\textnormal{o} & \ldots & {}^{(p-1)}\pi_\textnormal{o}
    \end{bmatrix}\, ,
    \label{eq:matrixPi_0}
\end{equation}
which is block upper triangular, with diagonal blocks of dimensions $1\times1$ and $r\times r$. Since the first diagonal block is the nonzero scalar $1$, and the vectors $
{}^1\pi_\textnormal{o},\,
{}^2\pi_\textnormal{o},\,
\dots,\,
{}^{(p-1)}\pi_\textnormal{o}
$
are linearly independent, then the $\Pi$ matrix is full rank. 
}
\end{remark}

\hyperref[goal:g2]{G2} is addressed by the following theorem.
\begin{theorem}\label{thm:goal2}
     Let us consider the system \eqref{eq:linksdynamics}-\eqref{eq:motorsdynamics}, under assumptions \hyperref[hp:sea]{H1}--\hyperref[hp:diagandconstmat]{H2}, and its linear formulation w.r.t. the dynamic parameters in \eqref{eq:hybridreglink}-\eqref{eq:hybridregmotor}.

    Consider a set of $p$ heterogeneous known payloads such that condition \eqref{eq:numberobjects} holds.

    The feedforward generalized control action required to track the desired trajectory $\hat{q}(\cdot): \left[0,\, t_\textnormal{f}\right] \rightarrow \mathbb{R}^n$, when a new payload characterized by the dynamic parameter vector $\pi^{*} \in \mathbb{R}^{r}$ is attached to the manipulator end-effector, is given by:
    \begin{equation}\label{eq:torquewithreg}
        \tau^{*}(t) = T(t) \, \Pi^{-1} \, \begin{bmatrix} 1 \\ \pi^{*} \end{bmatrix}\,, \quad \rev{\forall t \in [0, t_\textnormal{f}]} \,,
    \end{equation}
    where $T(t) \in \mathbb{R}^{n \times p}$ is defined as in \eqref{eq:matrixT} and $\Pi \in \mathbb{R}^{(r+1) \times p}$ is defined as in \eqref{eq:matrixPi}, and $p=r+1$ as in \eqref{eq:numberobjects}.
\end{theorem}
\begin{proof}
Given that condition~\eqref{eq:numberobjects} holds, the regressor matrix $\bar{Y}_\textnormal{o}(t) \in \mathbb{R}^{n \times (r+1)}$ in~\eqref{eq:matrixform} can be estimated by inverting matrix $\Pi$ as follows:
    \begin{equation}\label{eq:regestimation}
        \bar{Y}_\textnormal{o}(t) = T(t) \ \Pi^{-1} \,, \quad \rev{\forall t \in [0, t_\textnormal{f}]}.
    \end{equation}
    Combining~\eqref{eq:torqueregobji} and~\eqref{eq:regestimation} yields the formula for computing the torques necessary to track the desired trajectory with any new object.
    Specifically, given a new payload with dynamic parameter vector $\pi^{*} \in \mathbb{R}^{r}$, the feedforward control action is computed as:
    \begin{equation}\label{eq:taustar}
        \tau^{*}(t) = \bar{Y}_\textnormal{o}(t) \ \bar{\pi}^{*}, \quad \rev{\forall t \in [0, t_\textnormal{f}]}\,,
    \end{equation}
    where $\bar{\pi}^{*}=\begin{bmatrix} 1 \\ \pi^{*}\end{bmatrix}$.
    Substituting \eqref{eq:regestimation} into \eqref{eq:taustar} leads to equation \eqref{eq:torquewithreg}, which concludes the proof.
\end{proof}

\begin{remark}
    The results presented in Theorems~\ref{thm:goal1} and~\ref{thm:goal2} hold even when \(p > r+1\).
    In this case, the regressor matrix $\bar{Y}_\textnormal{o}(t)$ in \eqref{eq:matrixform} can be estimated using one of the right-inverse matrices of $\Pi$. For instance, one can use the Moore–Penrose pseudo-inverse $\Pi^{+} = \Pi^\top(\Pi\;\Pi^\top)^{-1}$, which solves the least squares problem. This leads to
    \begin{equation}\label{eq:pinv}
        \bar{Y}_\textnormal{o}(t) = T(t) \, \Pi^{+}\,, \quad \rev{\forall t \in [0, t_\textnormal{f}]}.
    \end{equation}
    Then, the generalized torques for any new object with parameters vector $\pi^{*}$ can be computed by combining \eqref{eq:taustar} and \eqref{eq:pinv}.
\end{remark}

%%%%%%%%%%%%%%%%%%%%%%%%%%%%%%%%%%%%%%%%%%%%%%%%%%%%%%%%%%%%%%%%%%%%%%%%%%%%%%%%%%%%%%%%%%%%%%%%%%%%%%%%%%%%%%%%%%%%%%%%%%%%%%%
%%%%%%%%%%%%%%%     Sensitivity Analysis      %%%%%%%%%%%%%%%
%%%%%%%%%%%%%%%%%%%%%%%%%%%%%%%%%%%%%%%%%%%%%%%%%%%%%%%%%%%%%%%%%%%%%%%%%%%%%%%%%%%%%%%%%%%%%%%%%%%%%%%%%%%%%%%%%%%%%%%%%%%%%%%
\paragraph*{Sensitivity analysis.} 
The proposed generalization method relies on the inversion of the parameters matrix $\Pi$ defined in~\eqref{eq:matrixPi} and on the collection of example torques in $T(t)$ defined in~\eqref{eq:matrixT}. 
To analyze how uncertainties and noise propagate to the generalized torques in \eqref{eq:torquewithreg}, the following perturbed matrices are introduced:
\begin{equation}
    \hat{T}(t) = T(t) + \Delta T(t) \,,
    \qquad 
    \hat{\Pi} = \Pi + \Delta \Pi \,,
\end{equation}
where $\Delta T(t)$ models perturbations due to sensor noise, control inaccuracies, and unmodeled dynamics in the torque measurements, while $\Delta \Pi$ represents uncertainties in the payload parameters arising, for instance, from manufacturing tolerances.
Using the first-order perturbation approximation of matrix inversion~\citep{golub1996van}, $\hat{\Pi}^{-1} \approx \Pi^{-1} - \Pi^{-1} \Delta \Pi \, \Pi^{-1}$, for $\| \Pi^{-1} \Delta \Pi \| < 1$, and neglecting higher-order terms, the resulting generalized torque perturbation can be approximated as

\begin{align}\label{eq:sensanalysis}
        \|\Delta\tau^*(t)\| \lesssim &\|\Delta T(t)\| \|\Pi^{-1}\| \| \bar{\pi}^*\| \nonumber \\&+ \|T(t)\| \|\Pi^{-1}\|^2 \|\Delta \Pi\| \| \bar{\pi}^*\| \,.
\end{align}

The expression derived in \eqref{eq:sensanalysis} highlights a distinct sensitivity of the method with respect to different uncertainty sources. Indeed, while the perturbation in torque measurements $\Delta T(t)$ propagates linearly through the norm of the inverse parameters matrix $\|\Pi^{-1}\|$, the uncertainty in payload parameters $\Delta \Pi$ depends on it quadratically.
This emphasizes the importance of Remark~\ref{remark1}: selecting heterogeneous example payloads is crucial to minimize the norm $\|\Pi^{-1}\|$, thereby ensuring the overall robustness of the estimated regressor $\bar{Y}_\textnormal{o}$.

%%%%%%%%%%%%%%%%%%%%%%%%%%%%%%%%%%%%%%%%%%%%%%%%%%%%%%%%%%%%%%%%%%%%%%%%%%%
%% Validation
%%%%%%%%%%%%%%%%%%%%%%%%%%%%%%%%%%%%%%%%%%%%%%%%%%%%%%%%%%%%%%%%%%%%%%%%%%%
\section{Validation}
\label{sec:validation}
This section presents the simulation and experimental results to validate the proposed method. 
Simulation tests are conducted on two systems: a 3-DoF soft articulated manipulator with each joint actuated by SEAs (Section~\ref{subsec: simressoft}), and the Franka Emika Panda (Section~\ref{subsec:simresrigid}).
In addition, experiments are performed using the Panda manipulator (Section~\ref{subsec:exprigid}) and the Panda robot equipped with a soft wrist system (Section~\ref{subsec:exprigidsoft}) powered by a qbmove Advanced actuator~\citep{mengacci2021overcoming}.
\rev{To demonstrate that the proposed method is independent of the controller used to collect the example torques, it is evaluated under two distinct control paradigms.
For the soft articulated manipulator, a model-based computed torque controller is employed; in contrast, torques for the Panda robot and soft wrist are computed via model-free ILC.}

\paragraph*{Object description.}
The payload is modeled as a link attached to the manipulator's end-effector.
Therefore, all payload parameters are expressed w.r.t. the end-effector frame.
The same dynamic parameters used for the robot’s links are employed to represent the payload parameters vector, including mass, center of mass, and inertia~\citep{Siciliano}, i.e.,
\begin{equation}\label{eq:objdynparam}
    \begin{aligned}
        \pi_\textnormal{o} = \big[ & m_\textnormal{o}, \, m_\textnormal{o} l_{c_{ox}}, \, m_\textnormal{o} l_{c_{oy}}, \, m_\textnormal{o} l_{c_{oz}},\\
        & \hspace{1em} \hat{I}_{oxx}, \, \hat{I}_{oxy}, \, \hat{I}_{oxz}, \, \hat{I}_{oyy}, \, \hat{I}_{oyz}, \, \hat{I}_{ozz} \big]^\top
    \end{aligned}
\end{equation}
where $\pi_\textnormal{o}\in\mathbb{R}^{10}$ and $m_o l_{c_ox}$, $m_o l_{c_oy}$ and $m_o l_{c_oz}$ are the object's first-order moments w.r.t. $x,\, y$ and $z$ axes, respectively. 
The terms $\hat{I}_{oij}$ represent the components of the object's second-order moment of inertia.
For the Franka Emika Panda robot, a set of cylindrical and parallelepiped payloads is considered.
A vertical grasping pose is assumed to ensure stability near the \SI{3}{kg} payload limit, 
\rev{with the object dimensions bounded by the gripper opening, i.e., \SI{80}{mm} for the Franka Hand in simulation, and \SI{45}{mm} for the soft claw used in the experiments.
Moreover, the object's principal axes are aligned with the end-effector frame, and its center of mass lies on the $z$ axis of the same frame.}
Consequently, \revii{the first-order moments along $x$ and $y$ and} the products of inertia are intrinsically null, allowing $\pi_\textnormal{o}$ to be reduced to a 5-element vector, i.e.,
\begin{equation}\label{eq:param_small}
    {}^i\pi_{\textnormal{o}} = \left[m_{i} \,, m_{i} l_{c_{iz}} \,, \hat{I}_{ixx} \,, \hat{I}_{iyy} \,, \hat{I}_{izz} \right]^\top,
\end{equation}
for a given object $i$. 
\revii{The discarded parameters do not contribute to the dynamics in this configuration, hence~\eqref{eq:param_small} is exact and avoids the ill-conditioning of $\Pi$.} 
\revii{Recalling Theorem~\ref{thm:goal1}, the dimension of the parameter vector determines the number of example objects.
The full 10-dimensional case~\eqref{eq:objdynparam} requires 11 objects, resulting in a matrix $\Pi \in \mathbb{R}^{11 \times 11}$ of the form~\eqref{eq:matrixPi}, whereas the reduced 5-dimensional case~\eqref{eq:param_small} requires 6 objects, yielding $\Pi \in \mathbb{R}^{6 \times 6}$.}
\revii{In the following, each payload is described either by~\eqref{eq:objdynparam} or by~\eqref{eq:param_small}, depending on the specific object set employed and on the scenario being addressed.}
\revii{In all cases, the condition number of $\Pi$ is reported to assess the numerical stability of the inversion procedure.}

\paragraph*{Evaluation metric.}
The Root Mean Square Error (RMSE) is employed to evaluate the tracking performance at the joints \[\text{RMSE} = \sqrt{\frac{1}{n} \sum_{i=1}^{n} \left( \hat{q}_i - q_i \right)^2}. \]

\paragraph*{\rev{Implementation.}}

\rev{Although the proposed method is derived in continuous time, its practical implementation in discrete time is straightforward and introduces no modification to the algorithm. Indeed, the proposed formulation is evaluated pointwise in time, and therefore the same algebraic expressions are simply computed at each controller sampling instant. Consequently, the regressor estimate and the generalized feedforward torques are evaluated at the controller update rate, which in the present implementation is \SI{1}{\kilo\hertz}. Both in simulation and experiments, joint position and velocity are directly measured by the robot. Joint acceleration, which is not directly available, is reconstructed by applying central finite differences to the measured velocity after low-pass filtering to reduce noise amplification. The only approximation introduced by the digital implementation concerns numerical differentiation, while the proposed payload generalization method itself is implemented exactly as derived in continuous time.}

%%%%%%%%%%%%%%%%%%%%%%%%%%%%%%%%%%%%%%%%%%%%%%%%%%%%%%%%%%%%%%%%%%%%%%%%%%%%%
\subsection{Simulation results: Soft articulated robot}\label{subsec: simressoft}
A 3-DoF soft articulated manipulator, whose structure is shown in Fig. \ref{fig:3Rsoft}, is simulated using MATLAB.
\begin{figure}%[b]
    \centering
    \includegraphics[width=0.55\linewidth]{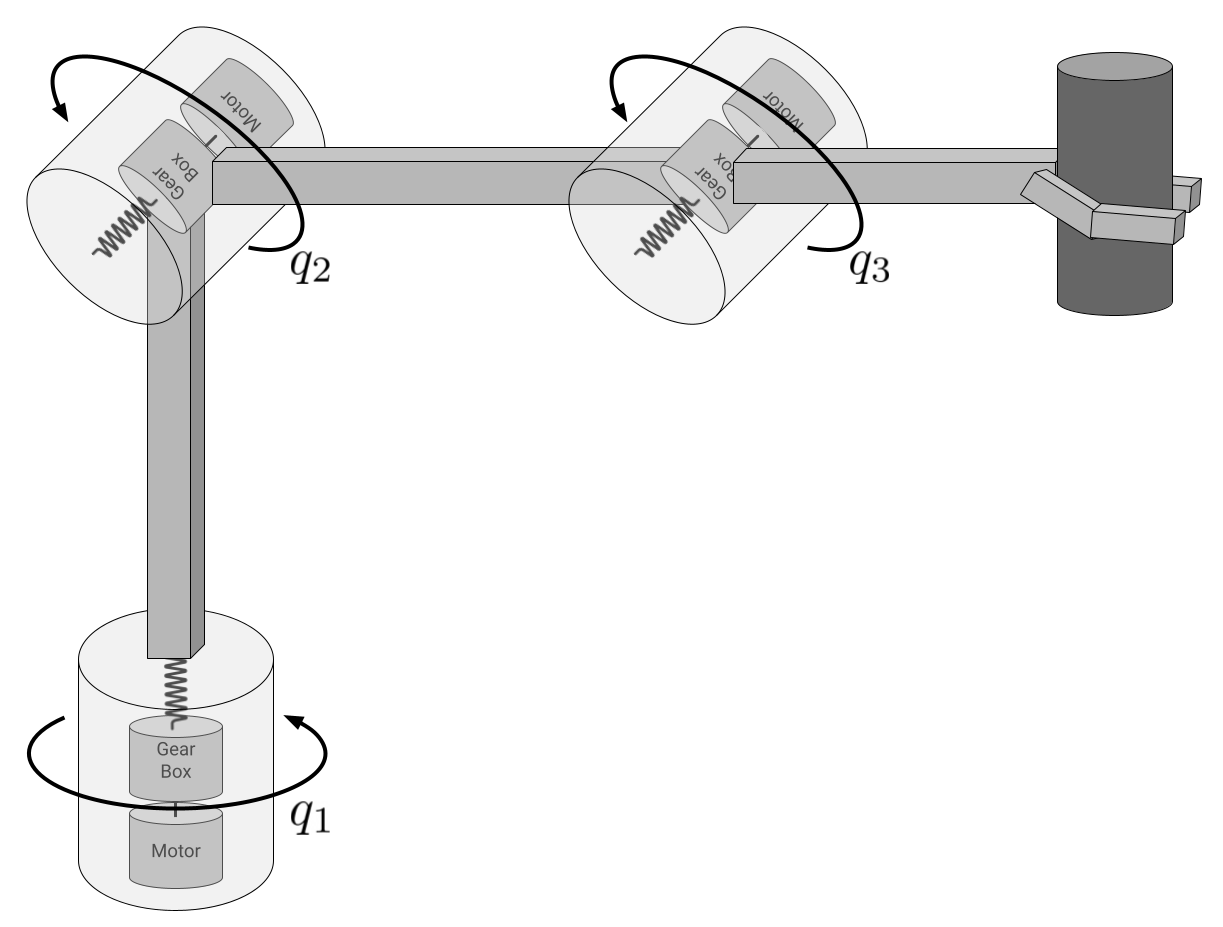}
    \caption{Schematic representation of the simulated 3-DoF soft articulated manipulator, actuated with SEAs.}
    \label{fig:3Rsoft}
\end{figure}
\revii{With reference to the model described in Section~\ref{subsec:system_model}, all the joints are elastic, i.e., $n=m=3$ and $\mathcal{E} = \{1, 2, 3\}$. 
Therefore, no column of $I$ is discarded in~\eqref{eq:stheta} and all the entries $s_i$ in~\eqref{eq:stau} are equal to one, hence $S_\theta = S_\tau = I_3$, the matrices $J$, $D$, and $K$ have full rank, and~\eqref{eq:linksdynamics}-\eqref{eq:motorsdynamics} reduce to the standard model of robots with elastic joints~\citep{Spong_1987}. 
}

\revii{Two studies are carried out on this system: the first addresses complex, eccentric, and arbitrarily oriented objects through the full ten-parameter vector~\eqref{eq:objdynparam}, whereas the second considers objects aligned with the end-effector frame, i.e., with the principal axes parallel to its axes and the center of mass on $z$, so that the reduced five-parameter vector~\eqref{eq:param_small} describes them exactly.
The control law and the desired trajectory described in the following are the same in both studies.}
The desired joint torques are computed using a pure feedforward computed torque controller \eqref{eq:torquesimsoft} to track the reference trajectory with \revii{the example objects of each study}.
This controller allows the proposed method to be evaluated in the presence of ideal tracking performance. 

The control law is defined as follows. 
The torque vector $\tau$ is expressed in terms of the link position vector $q$ and its time derivatives up to the fourth order~\citep{de2000feedforward}.
\rev{In the following, the first and second time derivatives of $q$ are denoted by $\dot{q}$ and $\ddot{q}$, while the superscripts $q^{(3)}$ and $q^{(4)}$ denote the third and fourth time derivatives, respectively.}
To simplify the notation, $N(q,\dot{q})\triangleq C(q,\dot{q})\dot{q}+G(q)$ is defined.
From~\eqref{eq:linksdynamics}, the expression of $\theta$ is derived and then differentiated w.r.t. time, leading to
\begin{equation}\label{eq:theta}
    \theta = q + K^{-1} (M(q)\ddot{q} + N(q,\dot{q}))\,,
\end{equation}
\begin{equation}\label{eq:dottheta}
    \dot{\theta} = \dot{q} + K^{-1} (M(q)\rev{q^{(3)}} + \dot{M}(q)\ddot{q} + \dot{N}(q,\dot{q}))\,,
\end{equation}
\begin{equation}\label{eq:ddottheta}
    \ddot{\theta} = \ddot{q} + K^{-1} (M(q)\rev{q^{(4)}} + 2\dot{M}(q)\rev{q^{(3)} +} \ddot{M}(q)\ddot{q} + \ddot{N}(q,\dot{q})) \,.
\end{equation}
Substituting~\eqref{eq:theta}--\eqref{eq:ddottheta} into the motor dynamics~\eqref{eq:motorsdynamics} yields the torque vector $\tau$ in terms of the configuration vector $q$ and its derivatives, i.e.,
\begin{align}\label{eq:torquesimsoft}
    \tau = &JK^{-1} \left( M(q) \rev{q^{(4)}} + 2\dot{M}(q) \rev{q^{(3)}} + (K+\ddot{M}(q))\ddot{q}\right) \nonumber \\ &+ JK^{-1} \left( \ddot{N}(q,\dot{q}) \right) 
    + DK^{-1} M(q) \rev{q^{(3)}} \nonumber \\ &+ DK^{-1} \left(\dot{M}(q) \ddot{q} + K \dot{q} + \dot{N}(q,\dot{q}) \right) \nonumber \\
    &+ M(q) \ddot{q} + N(q,\dot{q})\,,
\end{align}
which can be used to compute the desired torque using the desired trajectory and its derivatives.
\rev{Specifically, evaluating~\eqref{eq:torquesimsoft} requires this trajectory to be differentiable up to the fourth order, which is the only condition it must satisfy.}
\rev{As~\eqref{eq:torquesimsoft} shows, this control law requires the dynamics matrices of the system, and is therefore model-based. 
The payload generalization method of Section~\ref{sec:problemsolution}, instead, remains model-free, as it uses the example torques only as data, regardless of the controller that generated them. 
Hence, collecting the examples with a model-based controller does not affect the model-free nature of the generalization; when the system model is unavailable, a model-free controller can be used.}
The desired trajectory is
\begin{equation}
    \hat{q}_i(t) = \sum_{j=1}^{6} \left(\frac{c_j}{i}    \left(\frac{t}{t_\textnormal{f}}\right)^j\right)\,, \, \text{for } i = 1, \dots, 3.
\end{equation}
where $i$ denotes the manipulator joint index and $j$ is the summation index corresponding to the polynomial terms of the trajectory, each associated with a coefficient $c_j$.
The coefficients are $c_1=0$, $c_2 = 56.6893$, $c_3 = -305.0427$, $c_4 = 574.9919$, $c_5 = -461.6132$, $c_6 = 134.9746$ and $t_\textnormal{f} = \SI{10}{s}$. 
\revii{In each study, once the desired torques are obtained using~\eqref{eq:torquesimsoft}, the matrix $T$ is defined as in~\eqref{eq:matrixT}.}

\paragraph*{\revii{Full object parameter vector.}}
\revii{
The first study validates the method under the general payload description, in which no assumption is made on the position of the center of mass or on the orientation of the principal axes of each object, so that each payload is described by the full ten-dimensional parameter vector~\eqref{eq:objdynparam}, i.e., $r=10$.
\begin{table}[t]
    \centering
    \caption{\revii{Object parameter values for the full 10-parameter payload description of the 3-DoF soft articulated manipulator. For each object, $\|l_{c_i}\|$ is the distance of the center of mass from the origin of the end-effector frame, and the tilt is the rotation of the object principal axes w.r.t. the same frame.}}
    \label{tab:objectsetsimsoft10}
    \resizebox{\columnwidth}{!}{%
        \begin{tabular}{|c|c|c|c|c|}
            \hline
            \rowcolor{gray!40}
            Object [i]   & shape     & mass [\si{kg}]   & $\|l_{c_i}\|$ [\si{m}]  & tilt [\si{\degree}]  \\
            \hline
            \hline
            0  & --        & 0    & 0.00 & 0.0 \\
            \hline
            1  & box       & 3    & 0.17 & 40.8 \\
            \hline
            2  & cylinder  & 6    & 0.25 & 43.7 \\
            \hline
            3  & tube      & 9    & 0.32 & 72.6 \\
            \hline
            4  & ellipsoid & 12   & 0.21 & 95.2 \\
            \hline
            5  & cone      & 15   & 0.38 & 92.7 \\
            \hline
            6  & tri-prism & 18   & 0.31 & 144.8 \\
            \hline
            7  & L-shape   & 21   & 0.35 & 125.6 \\
            \hline
            8  & box       & 24   & 0.46 & 166.9 \\
            \hline
            9  & cylinder  & 27   & 0.39 & 148.3 \\
            \hline
            10 & ellipsoid & 30   & 0.38 & 131.6 \\
            \hline
            \rowcolor{gray!20}
            11 & tri-prism & 4.5  & 0.58 & 114.4 \\
            \hline
            \rowcolor{gray!20}
            12 & L-shape   & 25.5 & 0.41 & 112.7 \\
            \hline
        \end{tabular}%
        }
\end{table}
Accordingly, the example set comprises $p=11$ objects, while two further objects are used for testing.
These payloads are listed in Table~\ref{tab:objectsetsimsoft10}.
Each object is identified by a different number from 0 to 12, where Object 0 denotes the case in which the manipulator carries no payload, i.e., $^0\tau_\textnormal{o} = 0$ and~\eqref{eq:torqueobji} becomes $^0\tau = \tau_\textnormal{M}$.
In this simulation, Objects 0 to 10 are used to estimate the numerical regressor, while Objects 11 and 12 are used to test the generalized torques.
The objects have different shapes and are eccentric and arbitrarily oriented, so that all first-order moments and all products of inertia are non-zero.
The resulting matrix $\Pi \in \mathbb{R}^{11 \times 11}$ is full rank, with condition number $\kappa(\Pi)=2.2 \cdot 10^3$: the example payloads are thus far from linear dependence in the sense of Remark~\ref{remark1}, and~\eqref{eq:regestimation} can be evaluated reliably.
The method of Section~\ref{sec:problemsolution} is then applied to compute the generalized torques~\eqref{eq:torquewithreg} for the test objects.
\begin{figure}
    \centering
    \includegraphics[width=0.85\linewidth,trim={0 0.3cm 0 0},clip]{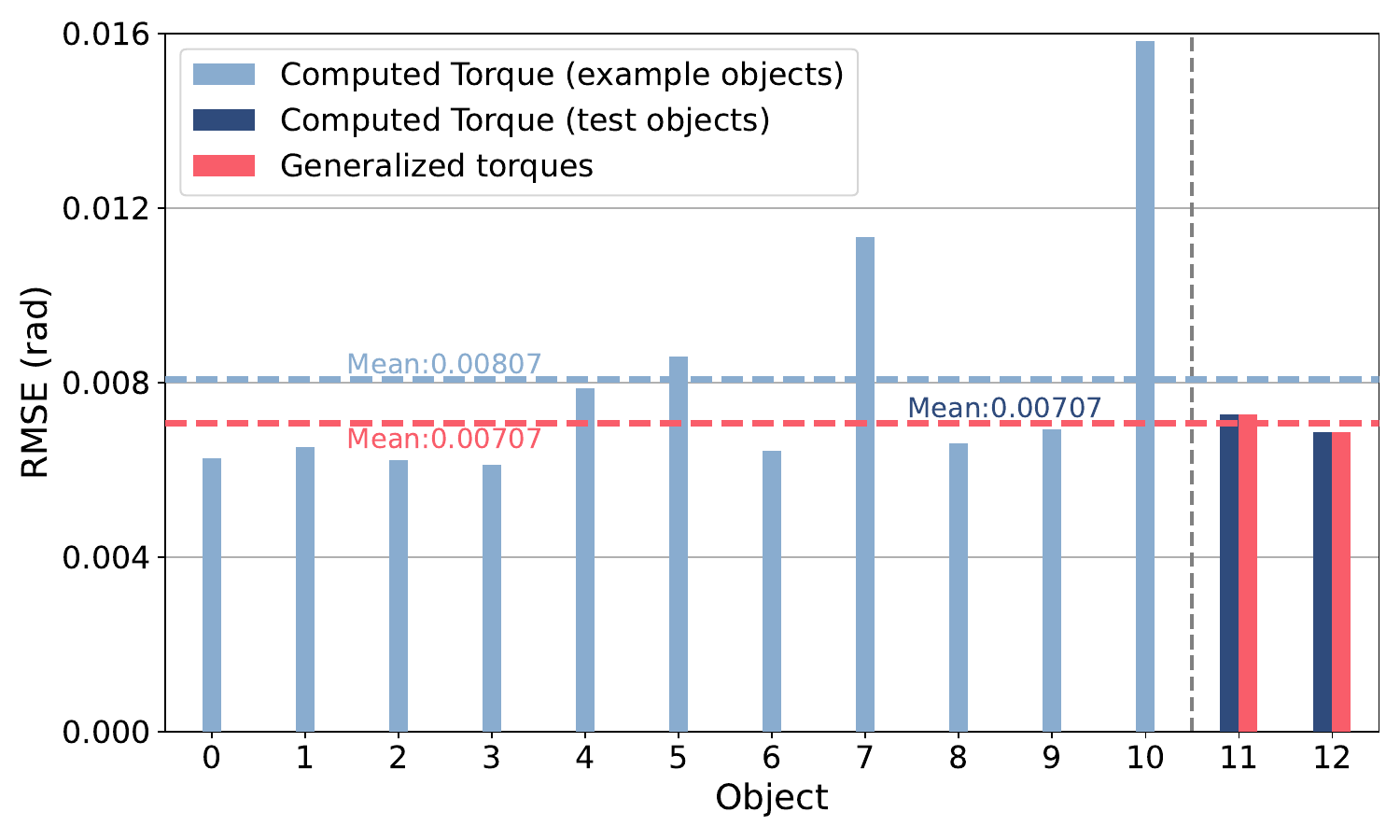}
    \caption{\revii{Simulation 3-DoF soft articulated manipulator with the full 10-parameter payload description: comparison of desired and generalized torques performance.}}
    \label{fig:3Rsoft_rmseoverobj10}
\end{figure}
Fig.~\ref{fig:3Rsoft_rmseoverobj10} compares the global RMSE of the generalized torques (red) with that of the computed torque controller, evaluated on the example objects (light blue) and on the test objects (dark blue).
}

\paragraph*{\revii{Reduced object parameter vector.}}
\revii{
The second study considers payloads whose center of mass lies on the $z$ axis of the end-effector frame and whose principal axes are aligned with the same frame, a configuration for which the reduced object description~\eqref{eq:param_small} is exact, i.e., $r=5$.
In this simulation, the example set reduces to $p=6$ objects, while two further objects are used for testing.
\begin{table}[t]
    \centering
    \caption{Object parameter values \revii{for the reduced 5-parameter payload description} of the 3-DoF soft articulated manipulator.}
    \label{tab:objectsetsimsoft}
    \resizebox{\columnwidth}{!}{%
        \begin{tabular}{|c|c|c|c|c|}
            \hline 
            \rowcolor{gray!40}
            Object [i]   & mass [\si{kg}]    & length [\si{m}]  & width [\si{m}]   & height [\si{m}]  \\
            \hline
            \hline
            0           & 0                 & 0.0                   & 0.0                   & 0.0                 \\
            \hline
            1           & 30                & 0.32                  & 0.16                  & 0.11                \\
            \hline
            2           & 10.82             & 0.15                  & 0.15                  & 0.05                \\
            \hline
            3           & 4.16              & 0.92                  & 0.46                  & 0.31                \\
            \hline
            4           & 22.67             & 0.47                  & 0.16                  & 0.24                \\
            \hline
            5           & 1.59              & 0.63                  & 0.32                  & 0.21                \\
            \hline
            \rowcolor{gray!20}
            6           & 19.35             & 0.29                  & 0.14                  & 0.09                \\
            \hline
            \rowcolor{gray!20}
            7           & 6.21              & 0.84                  & 0.28                  & 0.42                \\
            \hline
        \end{tabular}%
        }
\end{table}
The object set is listed in Table~\ref{tab:objectsetsimsoft}, where Objects 0 to 5 are the examples and Objects 6 and 7 are the test ones.
As in the previous study, Object 0 denotes the absence of payload.
}
Since the simulated objects are modeled as parallelepipeds, their centers of mass and their inertia matrices are computed using standard formulas.
\revii{Also in this study, the heterogeneity of the example payloads guarantees the invertibility of $\Pi \in \mathbb{R}^{6 \times 6}$, whose condition number $\kappa(\Pi)=1.0 \cdot 10^4$ keeps the estimation of the regressor through~\eqref{eq:regestimation} numerically well posed.
\begin{figure}
    \centering
    \includegraphics[width=0.85\linewidth,trim={0 0.3cm 0 0},clip]{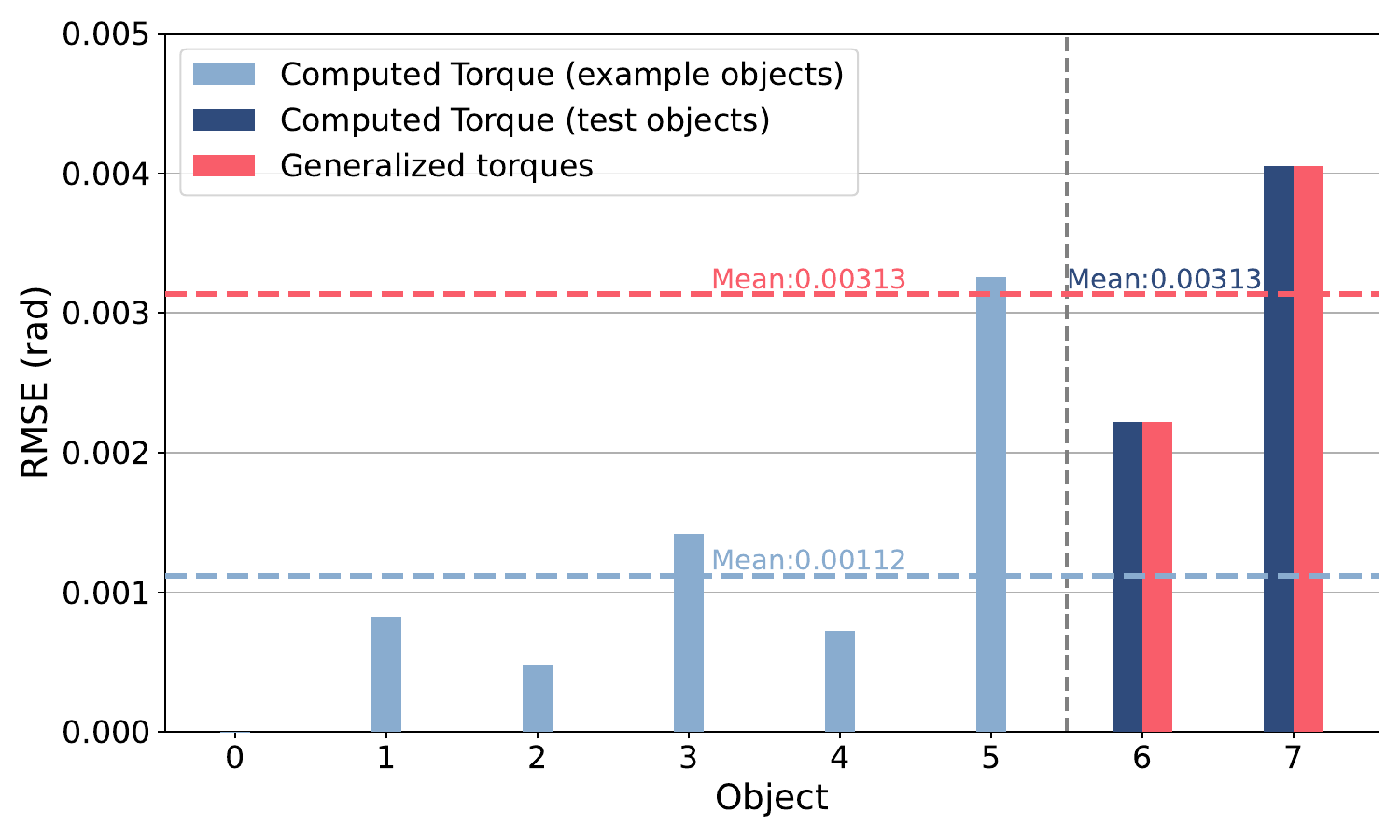}
    \caption{Simulation 3-DoF soft articulated manipulator \revii{with the reduced 5-parameter payload description}: comparison of desired and generalized torques performance.}
    \label{fig:3soft_rmseoverobj}
\end{figure}
The generalized torques~\eqref{eq:torquewithreg} are then obtained for Objects 6 and 7, and compared with the corresponding computed torque solution.
Fig.~\ref{fig:3soft_rmseoverobj} collects the RMSE over the whole object set: the light blue bars refer to the computed torque on the example objects, while the dark blue and the red ones compare the computed torque and the generalized torques on the test objects.
}

\begin{figure*}%[t]
    \centering
    \includegraphics[width=0.9\textwidth,trim={0 0.4cm 0 0},clip]{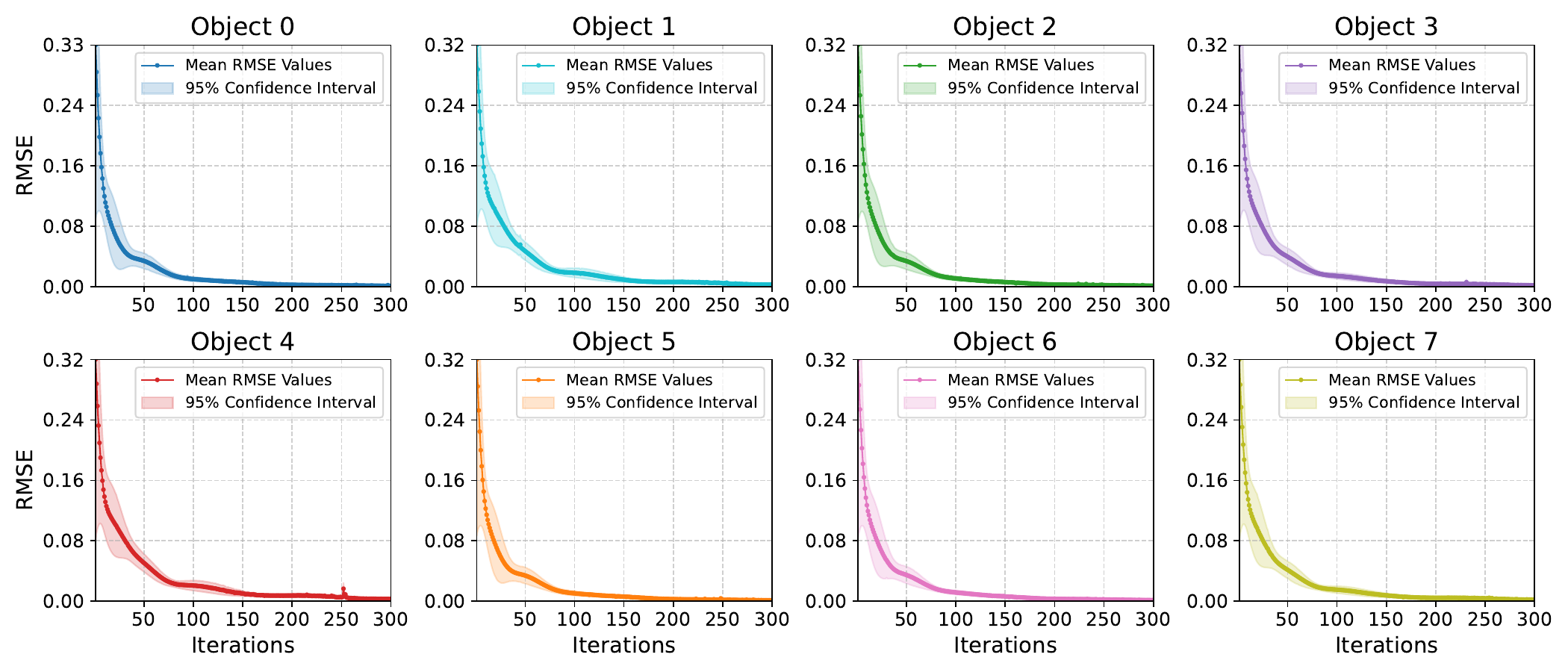}
    \caption{\rev{Simulation of Panda robot: the subplots show the ILC RMSE convergence over 300 iterations for the case of no payload (Object 0) and for all the remaining objects (Object 1-7).}}
    \label{fig:ilcsim_rmseoveriter}
\end{figure*}

\paragraph*{Discussion.}
Simulation results confirm that the proposed generalization method effectively enables tracking with new objects.
\revii{
This conclusion holds for both studies, since the generalized torques attain, on the test objects, the same RMSE as the computed torque controller employed to collect the example torques.
This is the expected outcome since the simulation is ideal and disturbance free, the example torques are exact, and the matrix $\Pi$, in both studies, is well conditioned, so that the regressor is recovered accurately and the generalization introduces no additional error.
This is confirmed by Figs.~\ref{fig:3Rsoft_rmseoverobj10} and~\ref{fig:3soft_rmseoverobj}.
In Fig.~\ref{fig:3Rsoft_rmseoverobj10}, obtained with the full ten-dimensional description, the generalized torques (red) and the computed torque on the test objects (dark blue) reach the same RMSE on Objects 11 and 12, with a mean of 0.00707 against a mean of 0.00807 over the example objects (light blue).
The method is thus successfully validated on eccentric and arbitrarily oriented payloads, for which all the first-order moments and all the products of inertia are non-zero.
The same holds in Fig.~\ref{fig:3soft_rmseoverobj} for the reduced five-dimensional description, where the red and the dark blue bars coincide on Objects 6 and 7, with a mean of 0.00313 against a mean of 0.00112 over the example objects (light blue).
The two simulations cannot be directly compared, as they refer to two different object sets, characterized by different features and, therefore, described by different parameter vectors. 
However, both figures show that the RMSE level depends on the specific payload and on the tracking accuracy of the controller used to collect the examples, and not on the generalization step.
}

These findings indicate that the proposed method successfully extracts the relevant information from the example torques in order to replicate the desired trajectory for any new object.  
It is interesting to point out that \rev{the generalized torques for the new objects are obtained without any knowledge of the system model, with the model-based controller used only to gather the example torques. Moreover, this holds even in the presence of joint elasticity.}

%%%%%%%%%%%%%%%%%%%%%%%%%%%%%%%%%%%%%%%%%%%%%%%%%%%%%%%%%%%%%%%%%%%%%%%%%%%%%
\subsection{Simulation results: Rigid system}\label{subsec:simresrigid}

To further evaluate the effectiveness of the proposed payload generalization method, the Franka Emika Panda is simulated via the Robot Operating System (ROS) in conjunction with the Franka Control Interface (FCI).
\revii{In contrast to the soft articulated case presented in Section~\ref{subsec: simressoft}, all the joints are rigid, i.e., $n = 7$ and $m = 0$ in the model of Section~\ref{subsec:system_model}: the index set $\mathcal{E}$ is empty, hence $S_\theta$ in~\eqref{eq:stheta} has no columns and all the entries $s_i$ in~\eqref{eq:stau} vanish, giving $S_\tau = 0$ together with $J = D = K = 0$. Consequently, the motor dynamics~\eqref{eq:motorsdynamics} is not present, and~\eqref{eq:linksdynamics} reduces to the classical rigid-robot dynamics $M(q)\ddot{q} + C(q, \dot{q})\dot{q} + G(q) = \tau$.}
In this setup, the torques necessary for tracking a desired trajectory with example objects are determined using ILC.
Specifically, a model-free ILC is implemented, which leverages only the torques computed at the previous iteration, together with the measured position, velocity, and acceleration errors, i.e.,
\begin{equation}\label{eq:ilclaw}
    \tau_{k+1}(t) = \tau_k(t) + \Upsilon  (K_\textnormal{P}  e_k(t) + K_\textnormal{V} \dot{e}_k(t) + K_\textnormal{A}  \ddot{e}_k(t))\,,
\end{equation}
where $(\cdot)_k$ indicates the iteration index, $K_\textnormal{P},\, K_\textnormal{V},\,K_\textnormal{A}\in\mathbb{R}^{n\times n}$ are the diagonal positive definite control gain matrices, and $\Upsilon\in\mathbb{R}^{n\times n}$ is the so-called learning gain, whose role is fundamental to ensure convergence \citep{pierallini2023iterative}.

From the starting configuration, i.e., $q(0)=q_\textnormal{0}$ at the initial time, up to the desired configuration $q(t_\textnormal{f})=q_\textnormal{f}$ at the final time $t_\textnormal{f} = \SI{5}{\second}$, the desired trajectory is a minimum jerk, i.e.,

{\small
\begin{equation}
\label{eq:min_jerk}
{
        \hat{q}(t) = q_\textnormal{0} + \left(q_\textnormal{f} - q_\textnormal{0}\right) \left(6 \left(\frac{t}{t_\textnormal{f}}\right)^5 - 15 \left(\frac{t}{t_\textnormal{f}}\right)^4 + 10 \left(\frac{t}{t_\textnormal{f}}\right)^3\right),       
}
\end{equation} 
}
where
$q_\textnormal{0}=\left[-1.25, -0.66, -0.69, -2.17, -0.26, 1.5,  0.063\right]^\top$ $\si{rad}$, and $q_\textnormal{f}=\left[0.31, -0.45, -0.39, -1.8, -0.5, 2.0, -0.3\right]^\top$ \si{rad}.
\rev{The choice of the boundary configurations $q_\textnormal{0}$, $q_\textnormal{f}$, and of the desired trajectory does not affect the validity of the generalization.
Theorems~\ref{thm:goal1}--\ref{thm:goal2} hold for any $\hat{q}(t)$, provided the same trajectory is used to collect the example torques and to compute the generalized ones, ensuring $T(t)$ consistency across objects.
Consequently, the trajectory is determined solely by the ILC scheme requirements. 
The minimum-jerk trajectory~\eqref{eq:min_jerk} is adopted because its smooth profiles enable faster convergence and improved tracking.}
\begin{table}[b]
    \centering
    \caption{Objects parameters values for simulation tests on the Franka Emika Panda manipulator.}
    \label{tab:objectsetsimrigid}
    \resizebox{\columnwidth}{!}{%
        \begin{tabular}{|c|c|c|c|c|}
            \hline 
            \rowcolor{gray!40}
            Object [i]   & mass [\si{kg}]    & length [\si{m}]  & width [\si{m}]   & height [\si{m}]  \\
            \hline
            \hline
            0           & 0                 & 0                 & 0                 & 0                 \\
            \hline
            1           & 2.7               & 0.035                & 0.062                & 0.027                \\
            \hline
            2           & 0.6               & 0.073                & 0.025                & 0.054                \\
            \hline
            3           & 1.8               & 0.036                & 0.071                & 0.043                \\
            \hline
            4           & 3.0               & 0.065                & 0.022                & 0.057                \\
            \hline
            5           & 0.3               & 0.032                & 0.079                & 0.014                \\
            \hline
            6           & 0.9               & 0.047                & 0.068                & 0.023                \\
            \hline
            7           & 2.1               & 0.041                & 0.075                & 0.018                \\
            \hline
        \end{tabular}%
        }
\end{table}
\rev{The objects considered in these simulations are parallelepipeds, whose characteristics are listed in Table~\ref{tab:objectsetsimrigid}, and where Object 0 means no payload.
Each object is described using the five-dimensional parameter vector~\eqref{eq:param_small}.}
\rev{Differently from Section~\ref{subsec: simressoft}, in this section the proposed method is validated using different example-test splits.
Specifically, following the consideration reported at the beginning of Section~\ref{sec:validation}, the objects are split into 6 examples employed for the generalization and 2 test objects. 
This leads to $\binom{8}{2}=28$ possible combinations of the two test objects out of the eight. To refer to these different combinations, in the following the notation $\textnormal{T}i\textnormal{-}j$ is used, indicating that objects $i$ and $j$ are used as test objects while the remaining six form the example set.}
\revii{Across the 28 combinations, $\Pi$ is always full rank and $\kappa(\Pi)$ never exceeds $7.8 \cdot 10^{6}$, ensuring no example set is ill-conditioned.}

Using the ILC controller~\eqref{eq:ilclaw}, the desired torques to track the desired trajectory~\eqref{eq:min_jerk} with the example objects as payloads are determined.
After tuning, $\Upsilon= 0.1\,\mathbf{I}_n$ is selected, where $\mathbf{I}_n$
is the $n \times n$ identity matrix, and the diagonal elements of the gains matrices are defined in Table~\ref{tab:gainilcsim}.
\begin{table}[t]
    \centering
    \caption{ILC controller gains.}
    \label{tab:gainilcsim}
    \begin{tabular}{|>{\columncolor{gray!20}}c|c|c|c|c|c|c|c|}
        \hline
        Joint   &   1   &   2   &   3   &   4   &   5    &   6    &   7     \\  \hline
        $k_p$     &   10  &   20  &   20  &   10  &   4    &   4    &   4     \\  \hline
        $k_v$     &   1   &   2   &   2   &   1   &   0.4  &   0.4  &   0.4   \\  \hline
        $k_a$     &   0.1 &   0.2 &   0.2 &   0.1 &   0.04 &   0.04 &   0.04  \\  \hline
    \end{tabular}
\end{table}
\rev{The gains are tuned empirically so that a monotone RMSE decrease is preserved across iterations, and kept fixed across all objects and experiments.
The velocity and acceleration gains $K_{\textnormal{V}}$ and $K_{\textnormal{A}}$ are set smaller than the position gain $K_{\textnormal{P}}$, so that the noisier velocity and acceleration error terms contribute less to the update.}
The learning process is stopped after 300 iterations, \rev{once the RMSE plateaus.} 
\rev{The RMSE convergence w.r.t. the eight objects is illustrated in Fig.~\ref{fig:ilcsim_rmseoveriter}.}

Once the ILC torques are computed, \rev{for all example sets $\textnormal{T}i\textnormal{-}j$}, the matrix $T$ is defined as in~\eqref{eq:matrixT}. \rev{Theorem~\ref{thm:goal2} is then applied to derive the generalized torques needed to track the same trajectory with new objects as in~\eqref{eq:torquewithreg}.} 

\rev{The results of this study are reported in Fig.~\ref{fig:rmse_all_comb}.}
\begin{figure*}%[htbp]
    \centering
    \includegraphics[width=0.8\linewidth]{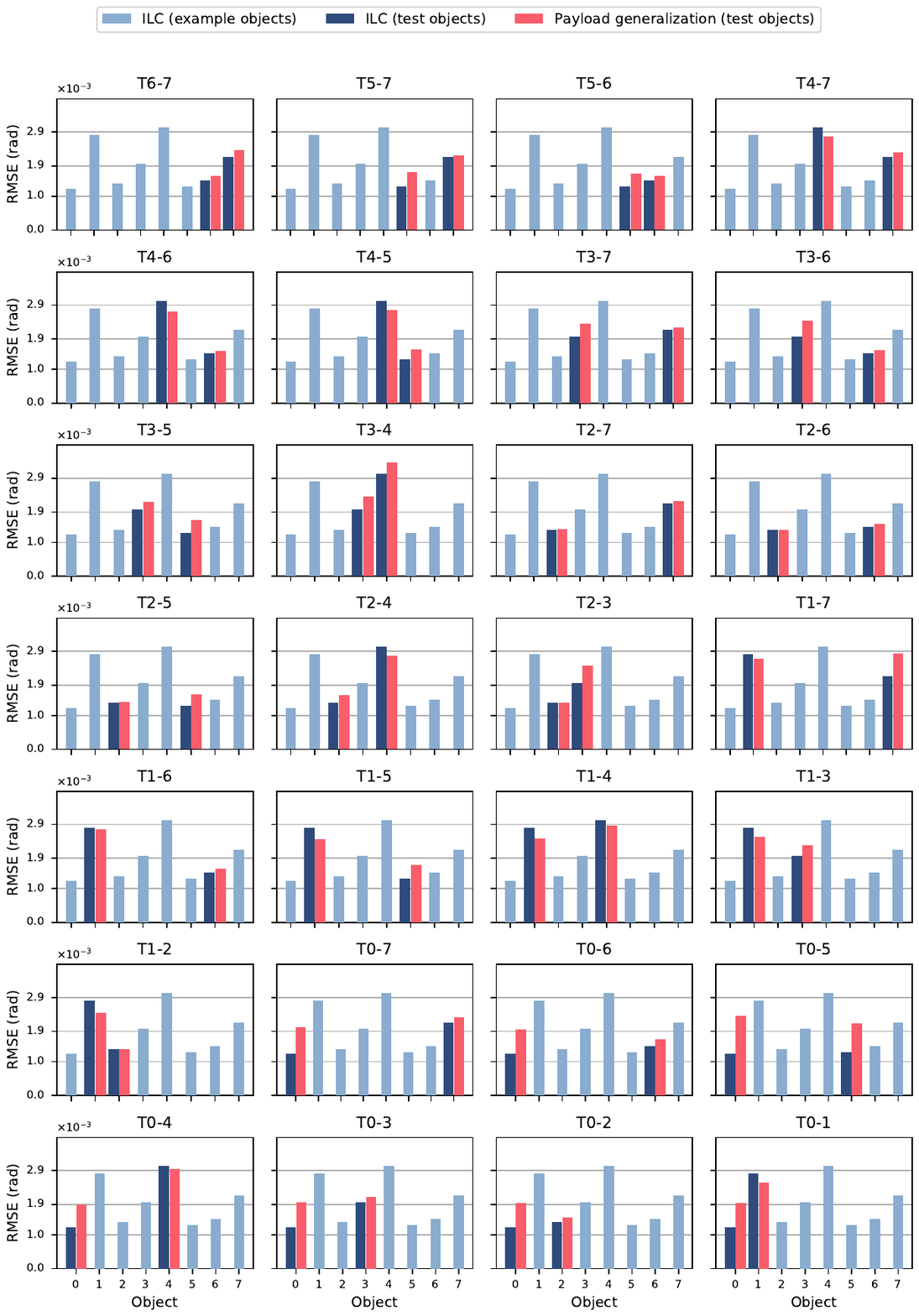}
    \caption{\rev{Tracking RMSE in Panda robot simulation over the $28$ example-test combinations. In each subplot, labeled $\textnormal{T}i\textnormal{-}j$, objects $i$ and $j$ are used as test objects while the remaining six form the example set; the light blue bars report the ILC RMSE on the six example objects, the dark blue bars the ILC RMSE on the two test objects, and the red bars the RMSE of the payload generalization on the same two test objects.}}
    \label{fig:rmse_all_comb}
\end{figure*}
\rev{Here, each subplot refers to one of the possible combinations of the example-test splits $\textnormal{T}i\textnormal{-}j$. 
Within each subplot the horizontal axis lists the eight objects and the vertical axis reports the average tracking RMSE: the light blue bars denote the ILC RMSE on the six example objects, the dark blue bars the ILC RMSE on the two test objects, and the red bars the RMSE of the payload generalization on the same two test objects.}

\rev{Across the 28 combinations, the red bars remain close to the dark blue ones and within the spread of the light blue bars, with no combination departing markedly from this range.
The largest discrepancy between the generalization RMSE and ILC RMSE is observed for the lightest objects 0 and 5, whereas on the heaviest objects 1 and 4 the generalization performs overall better than ILC.} 

\rev{Table~\ref{tab:rmse_std} summarizes the results of all 56 possible test cases obtained from the 28 example-test splits. Specifically, it compares the mean RMSE and standard deviation achieved by the ILC baseline and by the proposed payload generalization method. It is worth noting that the ILC results are identical whether they are used as training examples or as the ground truth.}

\begin{table}[h!]
    \centering
    \caption{\rev{Summary of mean RMSE and Standard Deviation (STD) across all example-test splits.}}
    \label{tab:rmse_std}
    \begin{tabular}{|>{\columncolor{gray!20}}l|c|c|}
        \hline
        \rowcolor{gray!20} & mean RMSE & STD \\  
        \hline
        ILC   &   0.001916   &   0.000658  \\  
        \hline
        Payload generalization     &   0.002108  &   0.000496   \\  
        \hline
    \end{tabular}
\end{table}

\begin{figure*}%[t]
    \centering
    \includegraphics[width=0.8\linewidth,trim={0 0.4cm 0 0},clip]{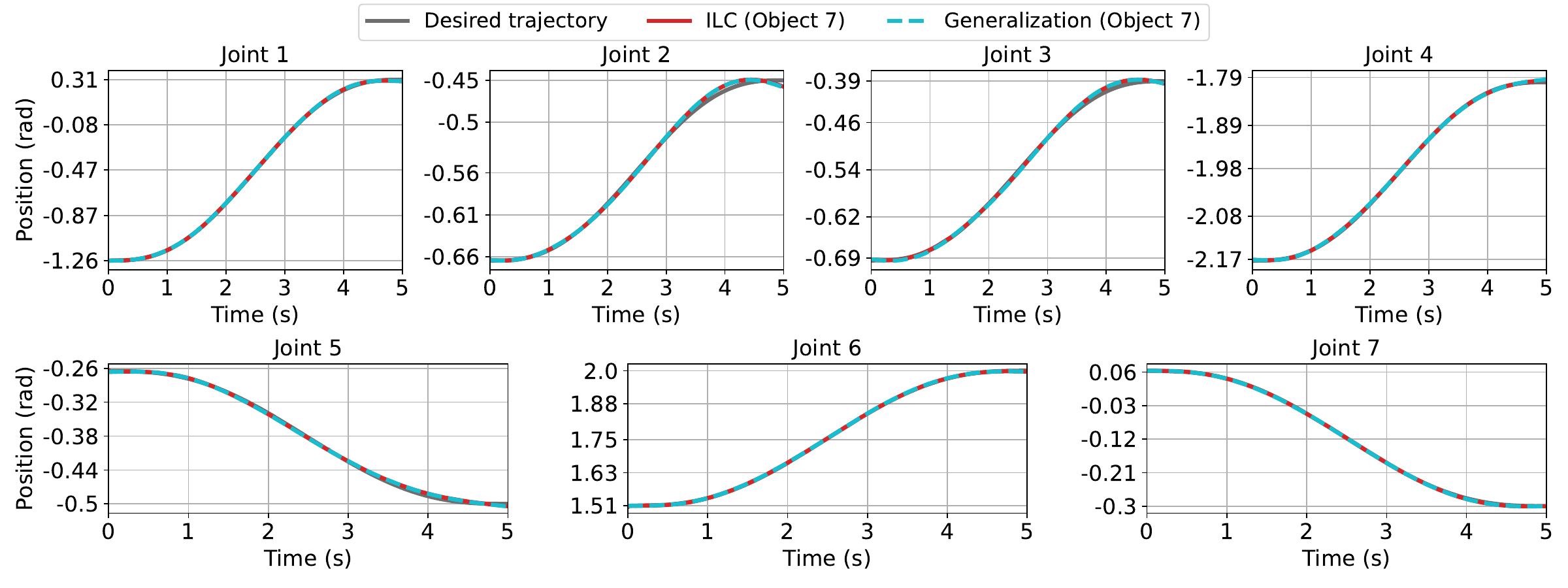}
    \caption{\rev{Simulation of Panda robot: comparison between the desired trajectory and the actual trajectories obtained with ILC and generalized torques on Object 7 for the example-test split T6-7.}}
    \label{fig:ilcsim_meas}
\end{figure*}

\rev{Fig.~\ref{fig:ilcsim_meas} compares the desired trajectory with the tracking performance achieved by ILC and the one achieved by the proposed payload-generalization approach. For the sake of space and readability, only one representative example is reported. 
Specifically, Object 7 in the T6-7 split is reported, where the generalized torques are computed from Objects 0 through 5 and evaluated on the held-out Objects 6 and 7.}

\rev{To quantify how the different splits influence the generalization performance as the object mass varies, the combinations are analyzed through two complementary metrics. 
Let $o \in \{1,\dots,7\}$ denote an object, and let $\mathcal{C}_o$ be the set of the $N_o = |\mathcal{C}_o|$ combinations in which $o$ appears as a test object.
For a given combination $c \in \mathcal{C}_o$ the following ratio is defined
\begin{equation}\label{eq:genstudy_rho}
    \rho = \frac{\sum_{c \in \mathcal{C}_o} \operatorname{RMSE}_{\mathrm{gen}}(o,c)}{\sum_{c \in \mathcal{C}_o} \operatorname{RMSE}_{\mathrm{ILC}}^{\mathrm{ex}}(c)}\,,
\end{equation}
which weights the generalized error on $o$, $\operatorname{RMSE}_{\mathrm{gen}}(o,c)$, against the ILC error on the corresponding example set, $\operatorname{RMSE}_{\mathrm{ILC}}^{\mathrm{ex}}(c)$, so that $\rho=1$ denotes parity.
The second metric is the relative error
\begin{equation}\label{eq:genstudy_delta}
    \delta = \frac{1}{N_o}\sum_{c \in \mathcal{C}_o} \frac{\operatorname{RMSE}_{\mathrm{gen}}(o,c) - \operatorname{RMSE}_{\mathrm{ILC}}(o)}{\operatorname{RMSE}_{\mathrm{ILC}}(o)}\,,
\end{equation}
which measures the deviation of the generalized error on test Object $o$ from the ILC error on that same object, $\operatorname{RMSE}_{\mathrm{ILC}}(o)$, so that $\delta=0$ denotes parity.
}

\begin{figure*}
    \centering
    % Prima sottofigura (a)
    \begin{subfigure}[b]{0.35\textwidth}
        \centering
        \includegraphics[width=\textwidth]{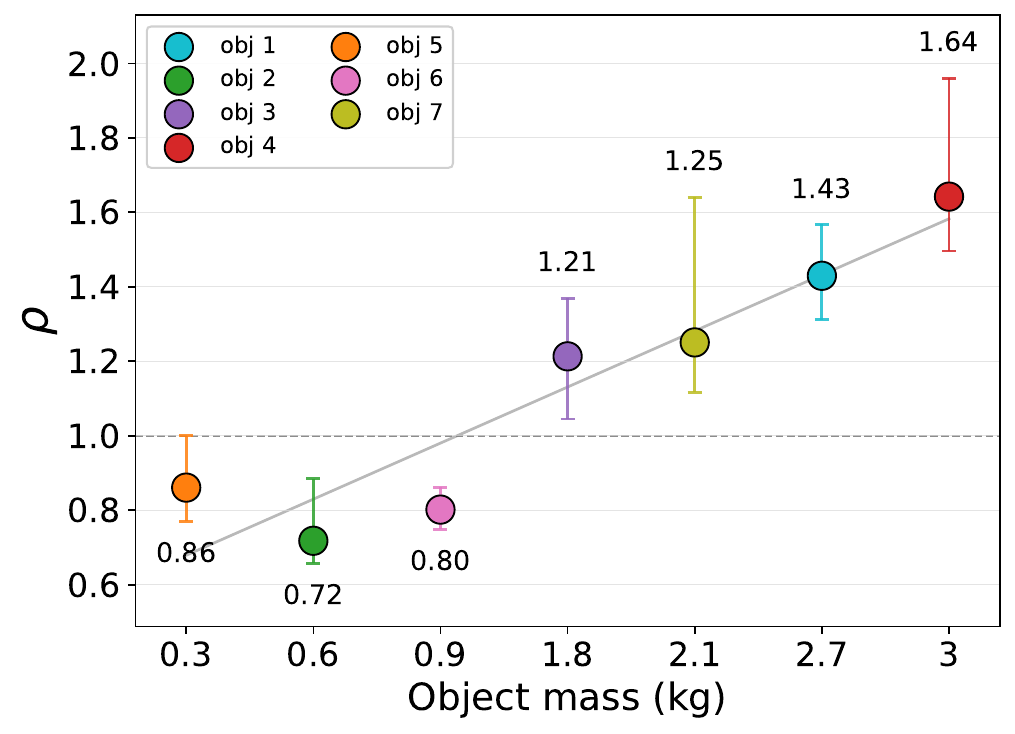}
        \caption{\rev{Ratio $\rho$ \eqref{eq:genstudy_rho}.}}
        \label{fig:genstudy_a}
    \end{subfigure}
    % \hfill
    \hspace{0.2cm}
    % Seconda sottofigura (b)
    \begin{subfigure}[b]{0.35\textwidth}
        \centering
        \includegraphics[width=\textwidth]{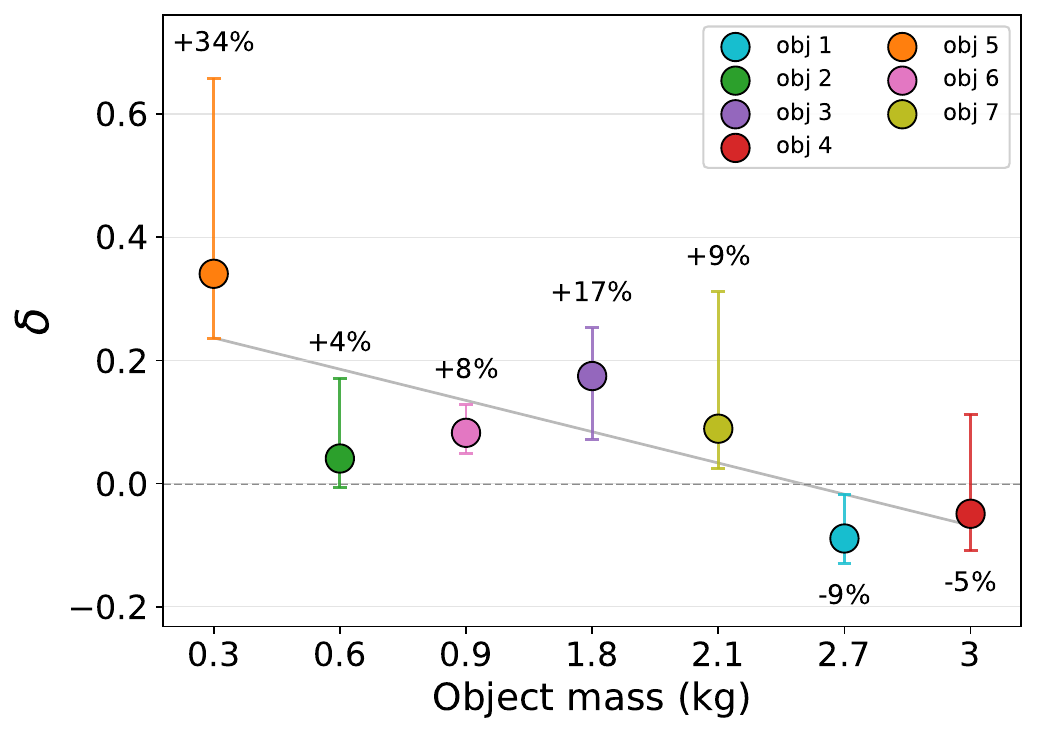}
        \caption{\rev{Relative error $\delta$ \eqref{eq:genstudy_delta}.}}
        \label{fig:genstudy_b}
    \end{subfigure}
    
    \caption{\rev{Generalization performance for the simulated Panda robot over the $28$ example-test combinations, with objects ordered by increasing mass. On the left (a), the ratio $\rho$ compares the generalization against the example-object set. On the right (b), the relative error $\delta$ quantifies the error introduced by the generalization relative to ILC on the same test object. \revii{The vertical bars span the minimum and maximum values attained over the $N_o=7$ splits in which the object belongs to the test set.} The dashed line marks parity ($\rho=1$, $\delta=0$) and the gray line is the linear fit versus mass.}}
    \label{fig:genstudy}
\end{figure*}

\rev{Both metrics are reported in Fig.~\ref{fig:genstudy}, with objects ordered by increasing mass.
\rev{In these results, Object 0 is not considered, as it represents an edge case corresponding to a zero-payload condition.}
In Fig.~\ref{fig:genstudy_a}, the ratio $\rho$ grows with object mass, staying below unity for the lighter objects 5, 2, and 6 and rising above it for the remaining heavier objects. 
In Fig.~\ref{fig:genstudy_b}, the relative error $\delta$ decreases with mass, reaching \revii{on average} negative values for the heaviest objects 1 and 4, and its maximum ($+34\%$) for Object 5.
}

\paragraph*{Discussion.}
Simulation results for the rigid manipulator confirm that the method effectively replicates the ILC tracking performance for novel payloads.
In particular, Fig.~\ref{fig:ilcsim_rmseoveriter} shows how the ILC controller effectively reduces the tracking error over the iterations, \rev{with each object converging to a different RMSE level according to its inertial properties.}
\rev{Building on these learned torques, the generalization is assessed over all 28 example-test splits in Fig.~\ref{fig:rmse_all_comb}: in every combination the generalized torques (red bars) reproduce the ILC performance on the test objects (dark blue bars) and stay within the error band of the example objects (light blue bars).
The results, therefore, do not depend on the specific split.}

\rev{However, two opposite scenarios, both involving objects at the extremes of the mass range, stand out in this figure.
For the lightest objects, 0 and 5, the generalized RMSE exceeds the ILC one. Since their masses are zero or negligible, the payload-dependent term of the regressor estimate, i.e., $Y_{\textnormal{o}}$, becomes ineffective, and the generalized torques reduce essentially to the intrinsic robot dynamics, discarding the load information that would improve tracking.
Conversely, for the heaviest objects, 1 and 4, the generalization outperforms ILC in every split where they act as test objects. 
However, their RMSE stays high in absolute terms, as these objects are the most demanding to track, so that the feedforward torques are predicted accurately yet the residual error remains larger than for the others.}

\rev{In general, the differences between the generalization RMSE and the ILC RMSE can be explained by the way the example torques are produced. 
Differently from the simulated soft articulated manipulator presented in Section~\ref{subsec: simressoft}, where the example torques are exact, being computed from the known model, here the example torques are learned by ILC in a more realistic simulation.
The ILC thus retains a small residual tracking error, which is inherited by the torque matrix $T(t)$, and hence by the generalized torques, explaining the slightly larger error with respect to ILC.
Nonetheless, as illustrated for the representative T6-7 split in Fig.~\ref{fig:ilcsim_meas}, the tracking performance achieved by the payload generalization method remains satisfactory, confirming the effectiveness of the approach.
}

\rev{The two metrics of Fig.~\ref{fig:genstudy} further explain these results in relation to the object masses.
The growth of $\rho$ with mass suggests that the generalization degrades on heavier payloads. However, the opposite trend of $\delta$ shows that, on those same objects, the generalization actually attains the accuracy of ILC.
Hence, although the average generalization RMSE is higher for the heaviest objects, it is nevertheless closer to the corresponding ILC value. 
The two readings are therefore complementary: together, they show that across all 28 splits the generalization remains comparable to ILC, matching or even improving upon it for the heavier objects and keeping the additional error bounded on the lighter ones, thus confirming the effectiveness of the method.
}
\revii{The vertical bars in Fig.~\ref{fig:genstudy} confirm that both trends with mass hold across the splits, the spread being largest for the lightest Object 5, whose $\delta$ ranges from $+23.5\%$ to $+65.8\%$.
The latter is attained in the split T0-5, where both the zero-mass and the lightest object are held out from the example set, consistently with the previous results.}
Notably, this level of performance is attained through feedforward torques, eliminating the requirement for repeated learning phases.
\rev{As a result, the generalization method allows the same trajectory to be tracked with novel objects, keeping the tracking error bounded within the RMSE range attained by ILC across the objects, with only a small gap consistent with the sensitivity analysis in~\eqref{eq:sensanalysis}.}

\rev{As reported in Table~\ref{tab:rmse_std}, the proposed approach introduces only a modest increase of approximately 10\% in the average tracking error with respect to ILC, while exhibiting lower variability across the different combinations. This result indicates that the generalized feedforward actions provide consistently accurate trajectory tracking independently of the specific selection of the example payloads, confirming the robustness and generalization capability of the proposed method. Considering that no learning or optimization is performed for the test payloads, these results demonstrate that the proposed method is able to reproduce, in a zero-shot fashion, the performance of an ILC controller that would otherwise require a dedicated learning phase for each new payload.}

%%%%%%%%%%%%%%%%%%%%%%%%%%%%%%%%%%%%%%%%%%%%%%%%%%%%%%%%%%%%%%%%%%%%%%%%%%%%%
\subsection{Experimental results: Rigid system}\label{subsec:exprigid}

Simulation tests for the rigid system are replicated here to validate the method in a real-world setup, following the steps in Section~\ref{subsec:simresrigid}\revii{, where the robot is again fully rigid, i.e., $n=7$ and $m=0$ in~\eqref{eq:linksdynamics}-\eqref{eq:motorsdynamics}}.
First, the desired torques are derived using ILC for the example objects; then, the generalized torques are computed to test the proposed method on new objects.
\begin{table}[b]
    \centering
    \caption{Parameter values of objects used in experiments with the Franka Emika Panda robot and the Panda robot equipped with a soft wrist.}
    \label{tab:objectsetrigidsys}
    \resizebox{\columnwidth}{!}{ % Resize to fit the column width
    \begin{tabular}{|c|c|c|c|c|c|}
        \hline
        \rowcolor{gray!40} 
        Object [i] & mass [\si{kg}]    & $cm_z$ [\si{mm}]  & $I_{xx}$ [\si{\g\meter\squared}]  & $I_{yy}$ [\si{\g\meter\squared}]  & $I_{zz}$ [\si{\g\meter\squared}]  \\ \hline\hline
        0          & 0                 & 0                 & 0                                 & 0                                 & 0                                 \\  \hline
        1          & 0.142             & 36                & 0.126                             & 0.126                             & 0.126                             \\  \hline
        2          & 1.399             & 74                & 2.65                              & 2.74                              & 0.284                             \\  \hline
        3          & 0.740             & 75                & 1.42                              & 1.42                              & 0.0710                            \\  \hline
        4          & 0.442             & 50                & 0.460                             & 0.559                             & 0.283                             \\  \hline
        5          & 0.942             & 50                & 0.840                             & 0.910                             & 0.180                             \\  \hline
        \rowcolor{gray!20} 
        6          & 0.307             & 49                & 0.258                             & 0.258                             & 0.0204                            \\  \hline
        \rowcolor{gray!20} 
        7          & 0.535             & 98                & 1.86                              & 1.86                              & 0.271                             \\  \hline
        \rowcolor{gray!20} 
        8          & 0.688             & 75                & 1.46                              & 1.46                              & 0.359                             \\  \hline
    \end{tabular}
}
\end{table}

\begin{figure}%[t]
    \centering
    \includegraphics[width=0.3\linewidth]{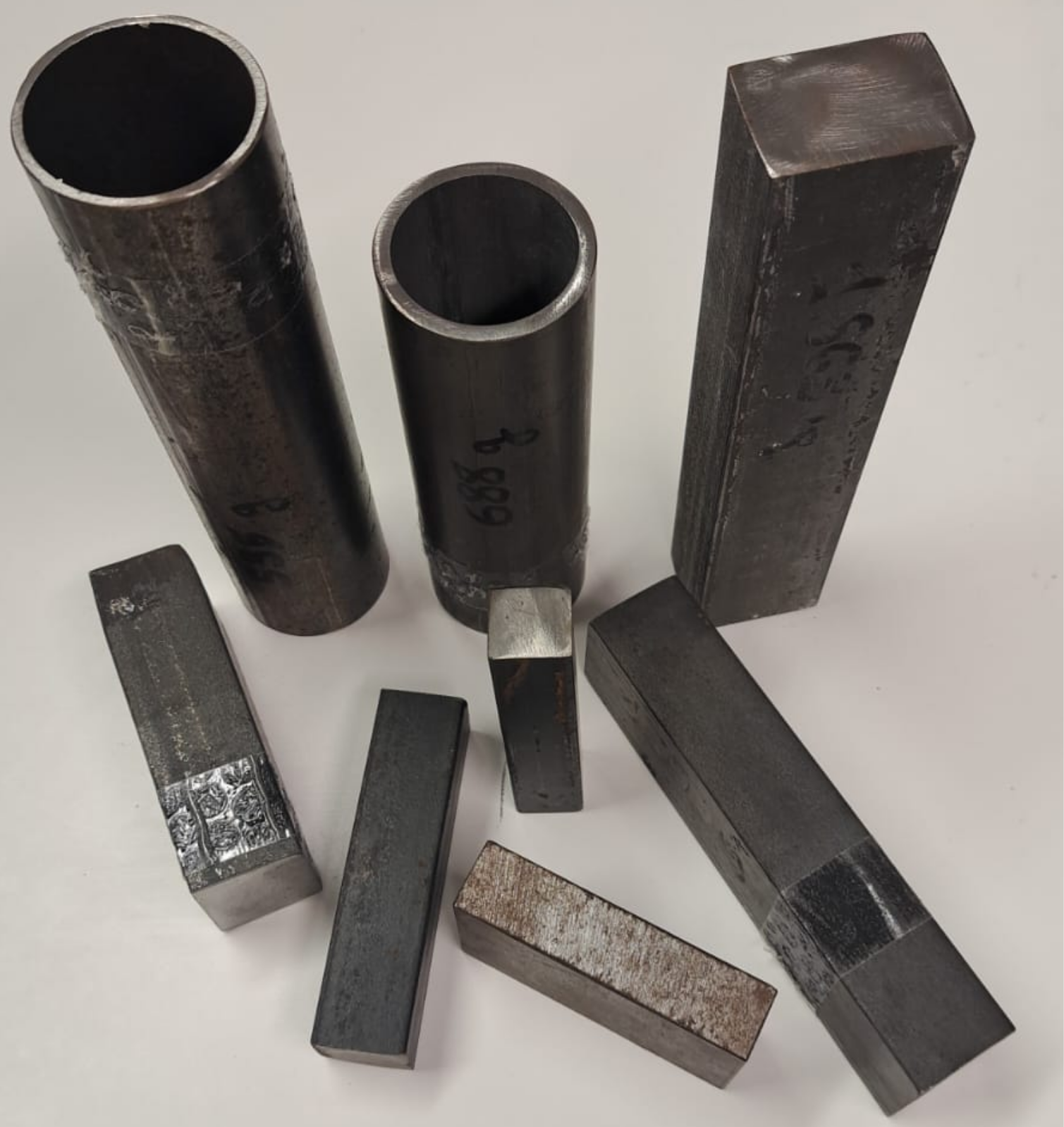}
    \caption{Objects set used for real-world experiments.}
    \label{fig:objects}
\end{figure}

The characteristics of the example objects are listed in Table~\ref{tab:objectsetrigidsys} and correspond to those shown in Fig.~\ref{fig:objects}.
These are regularly shaped objects available in the laboratory, whose simple geometry allows their inertial parameters to be roughly determined with standard measuring instruments and to adopt the reduced five-dimensional parameter vector~\eqref{eq:param_small}.
Objects 0 through 5 are used to estimate the numerical regressor, whereas objects 6, 7, and 8 are used to test the generalization procedure.
\revii{For this object set, $\Pi$ is again full rank, with $\kappa(\Pi)=7.7 \cdot 10^{5}$.}
The trajectory~\eqref{eq:min_jerk} moves the robot in $t_\textnormal{f} =$ \SI{10}{\second} from configuration 
$q_0 = \left[0.13, 0.18,  -0.63, -1.56, 0.13, 1.73, -1.80\right]^\top$ $\si{rad}$ to
$q_\textnormal{f} = \left[1.62, 0.61, -1.42, -2.02, 0.99, 1.15, 0.053\right]^\top$
$\si{rad}$.
The same ILC gains selected for simulation, detailed in Table~\ref{tab:gainilcsim}, are used.
During the experiment, the homing between iterations is performed using a computed torque controller.
The learning process is completed after 150 iterations.
The ILC performance on the six example objects is illustrated in Fig.~\ref{fig:ilcreal_rmseoveriter}.
\begin{figure*}%[t]
    \centering
    \includegraphics[width=0.8\linewidth,trim={0 0.4cm 0 0},clip]{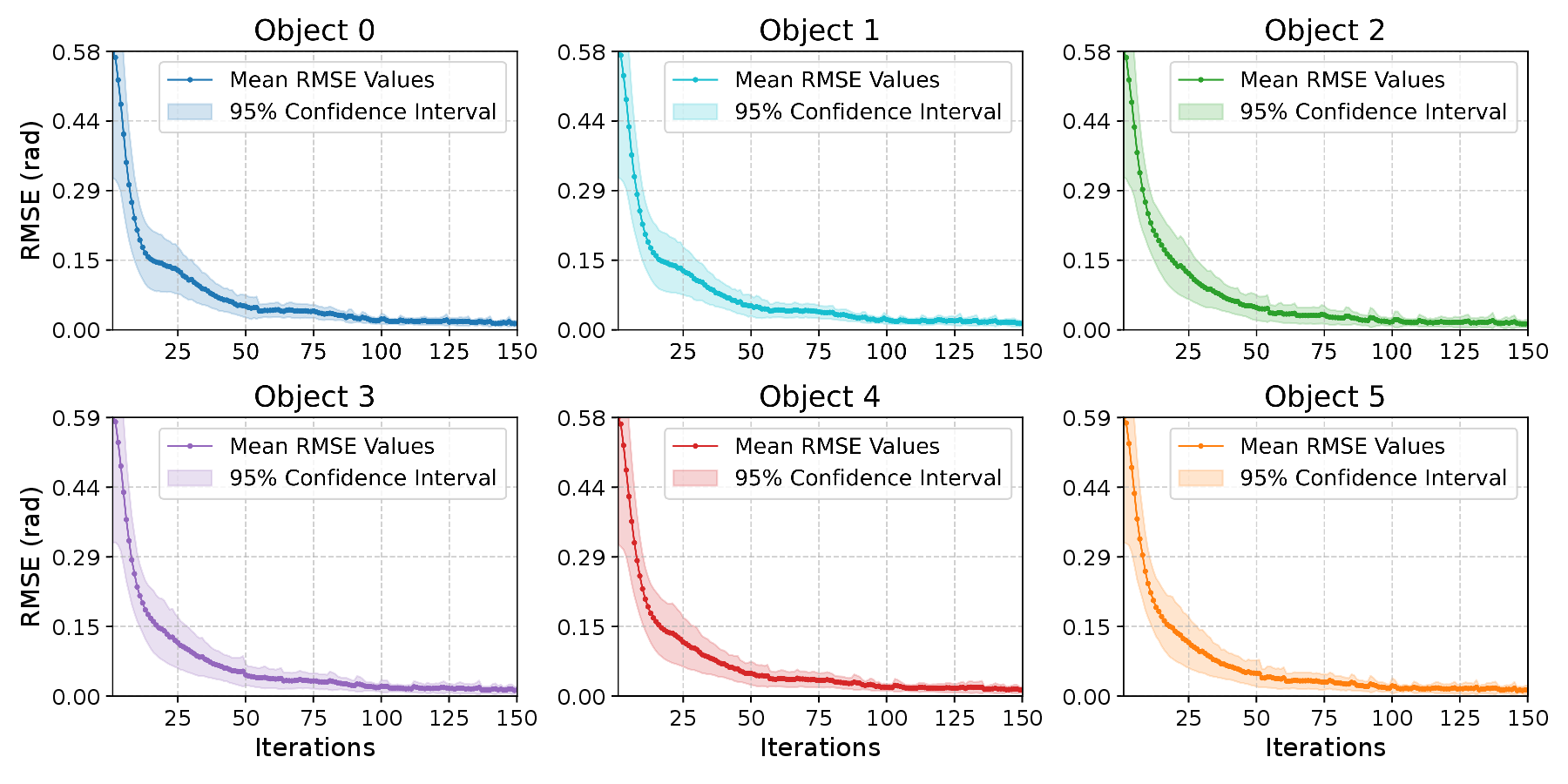}
    \caption{Experiment on real-world Panda robot: the subplots show the ILC RMSE reduction over 150 iterations for all six example objects.}
    \label{fig:ilcreal_rmseoveriter}
\end{figure*}
\begin{figure*}%[t]
    \centering
    \includegraphics[width=0.8\linewidth,trim={0 0.4cm 0 0},clip]{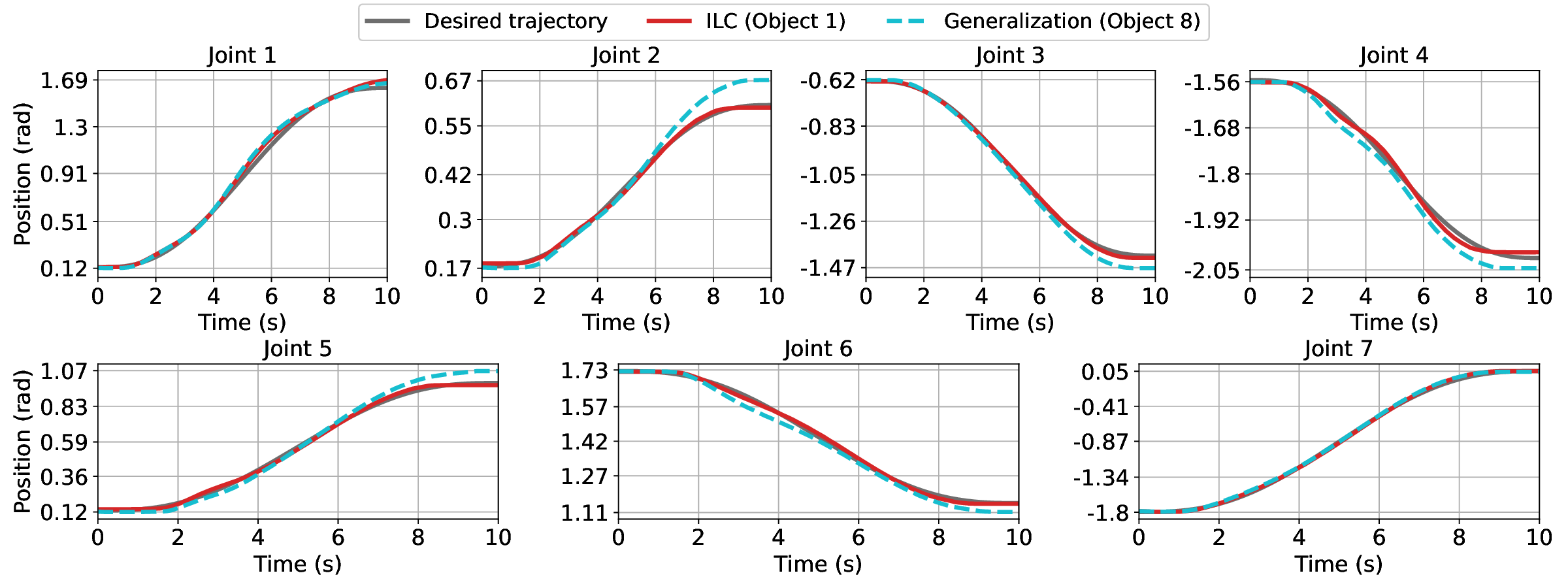}
    \caption{Experiment on real-world Panda robot: comparison between desired trajectory and actual trajectories obtained with ILC and generalized torques.}
    \label{fig:ilcreal_meas}
\end{figure*}
\begin{figure}%[b]
    \centering
    \includegraphics[width=0.85\linewidth,trim={0 0.3cm 0 0},clip]{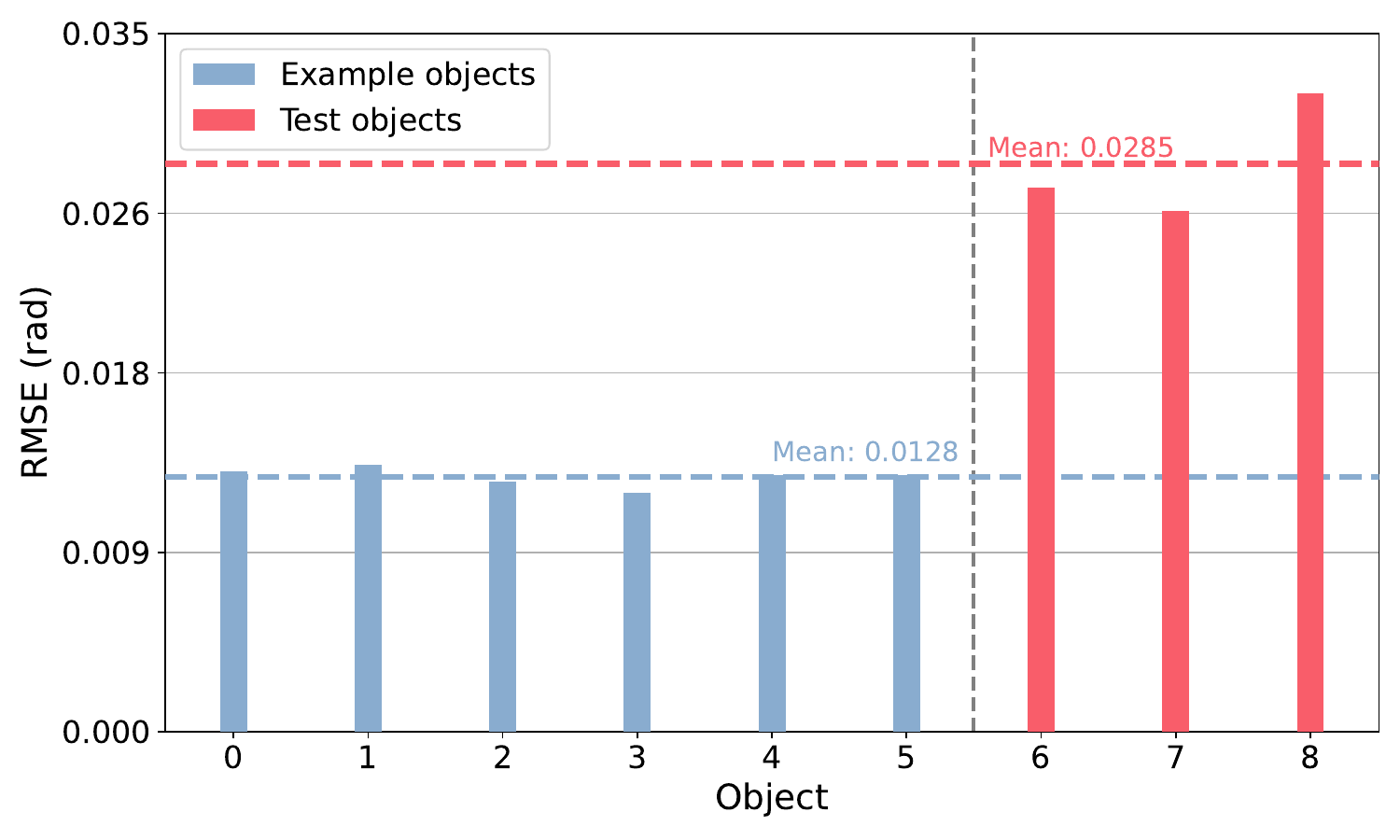}
    \caption{Experiment on real-world Panda robot: comparison of ILC-based and generalized torques performance.}
    \label{fig:ilcreal_rmseoverobj}
\end{figure}
After deriving the desired torques with ILC, the payload generalization procedure is applied to compute the feedforward control actions (as in~\eqref{eq:torquewithreg}) required to track the desired trajectory with test objects.

Fig.~\ref{fig:ilcreal_meas} illustrates the trajectory tracking performance for both ILC and generalization methods, showing the results for Object 1 and Object 8, respectively, compared to the desired trajectory.  Fig.~\ref{fig:ilcreal_rmseoverobj} summarizes the results for all objects.

\paragraph*{Discussion.}
The experiments conducted on the rigid system demonstrate how the generalized torques enable good tracking performance to be achieved even in real-world scenarios.
As with the simulation results, Fig.~\ref{fig:ilcreal_rmseoveriter} shows how the ILC controller effectively reduces position error over iterations, ensuring good trajectory tracking results.
In contrast to the simulations, Fig.~\ref{fig:ilcreal_rmseoverobj} shows that the generalization procedure, compared with the performance of ILC, introduces some additional position error\rev{: the mean RMSE grows from 0.0128 on the ILC example objects to 0.0285 on the generalized test objects}.
\rev{This roughly twofold increase originates from the two uncertainty sources isolated in~\eqref{eq:sensanalysis}: the torque measurement noise of the real actuators, which affects the estimated example-torque matrix $T(t)$, and the residual error in the only approximately known payload parameters, which enters through $\Pi$.
Because the example torques inherently embed the ILC residual and sensor noise, the observed performance gap represents the expected propagation of these errors.}
However, it is noteworthy that the RMSE across all objects remains bounded. \rev{This is confirmed by the joint trajectories in Fig.~\ref{fig:ilcreal_meas}, where both the ILC and the generalized motions follow the desired trajectory throughout the whole task: the deviations are localized in the final portion of the motion and concentrated on the joints undergoing the widest excursion, while the remaining joints are tracked almost exactly. Thus, the error never grows into a loss of tracking.}
Therefore, it can be concluded that trajectory tracking is successfully achieved for all three test objects, even in real-world experiments, indicating the effectiveness of the proposed method in generalizing torques for various payloads.

%%%%%%%%%%%%%%%%%%%%%%%%%%%%%%%%%%%%%%%%%%%%%%%%%%%%%%%%%%%%%%%%%%%%%%%%%%%%%
\subsection{Experimental results: Panda robot and soft wrist}
\label{subsec:exprigidsoft}

\begin{figure}%[hb!]
\begin{minipage}[b]{0.49\columnwidth}
    \centering
    \includegraphics[height=2.6cm,angle=90]{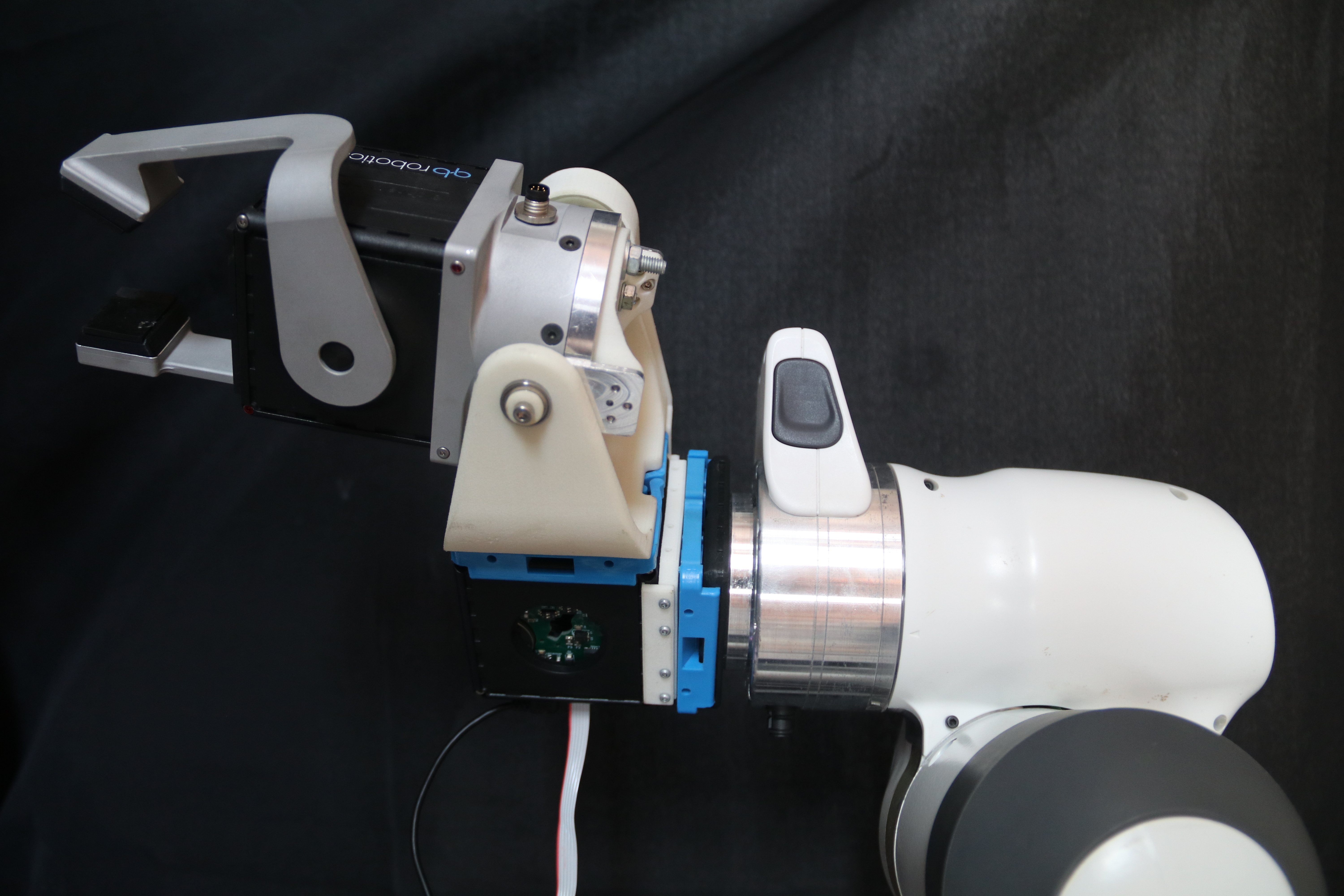}%2.6
\end{minipage}
\begin{minipage}[b]{0.49\columnwidth}
    \centering
    \includegraphics[height=3.9cm]{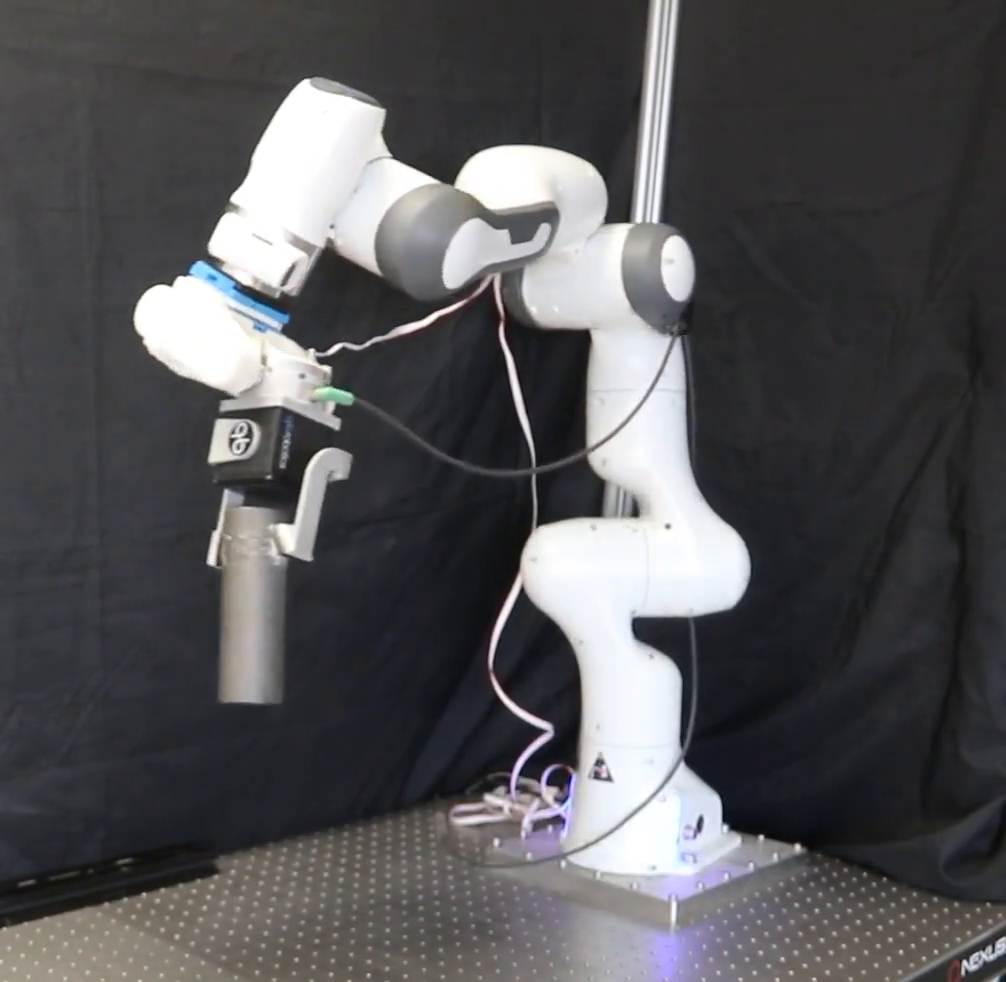}
\end{minipage}
    \caption{
    System setup for the final experiments: the Panda robot is equipped with an elastic wrist. On the left, the soft wrist is connected to the gripper for object grasping. On the right, the system performs trajectory tracking while holding an object.
    }
    \label{fig:softwrist}
\end{figure}

\begin{figure*}%[t]
    \centering
    \includegraphics[width=0.8\linewidth,trim={0 0.4cm 0 0},clip]{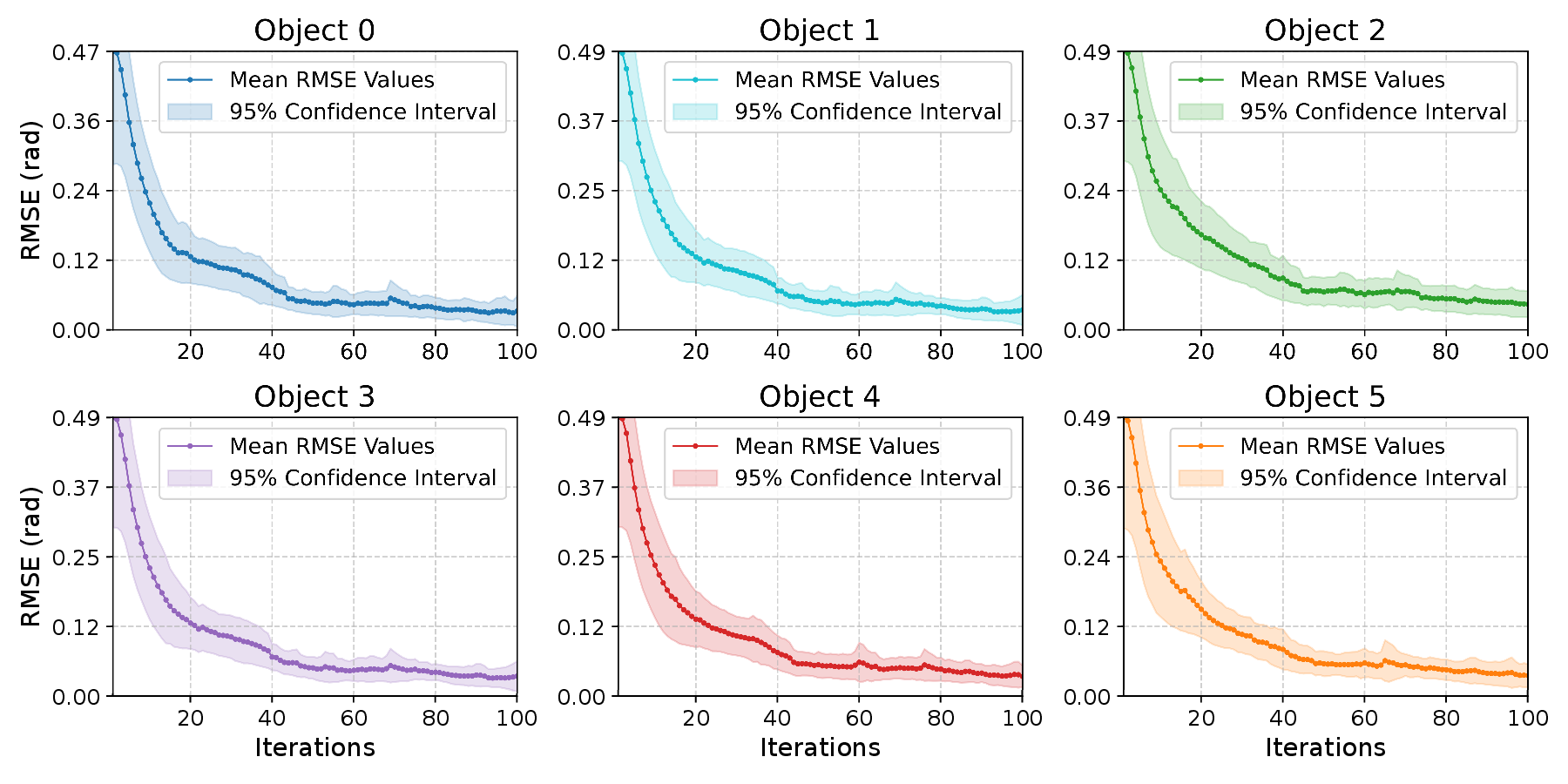}
    \caption{Experiment on Panda manipulator with elastic wrist: the subplots show the ILC RMSE reduction over 100 iterations for all six example objects.}
    \label{fig:ilchyb_RMSEoveriter}
\end{figure*}
\begin{figure*}
    \centering
    \includegraphics[width=0.8\linewidth,trim={0 0.4cm 0 0},clip]{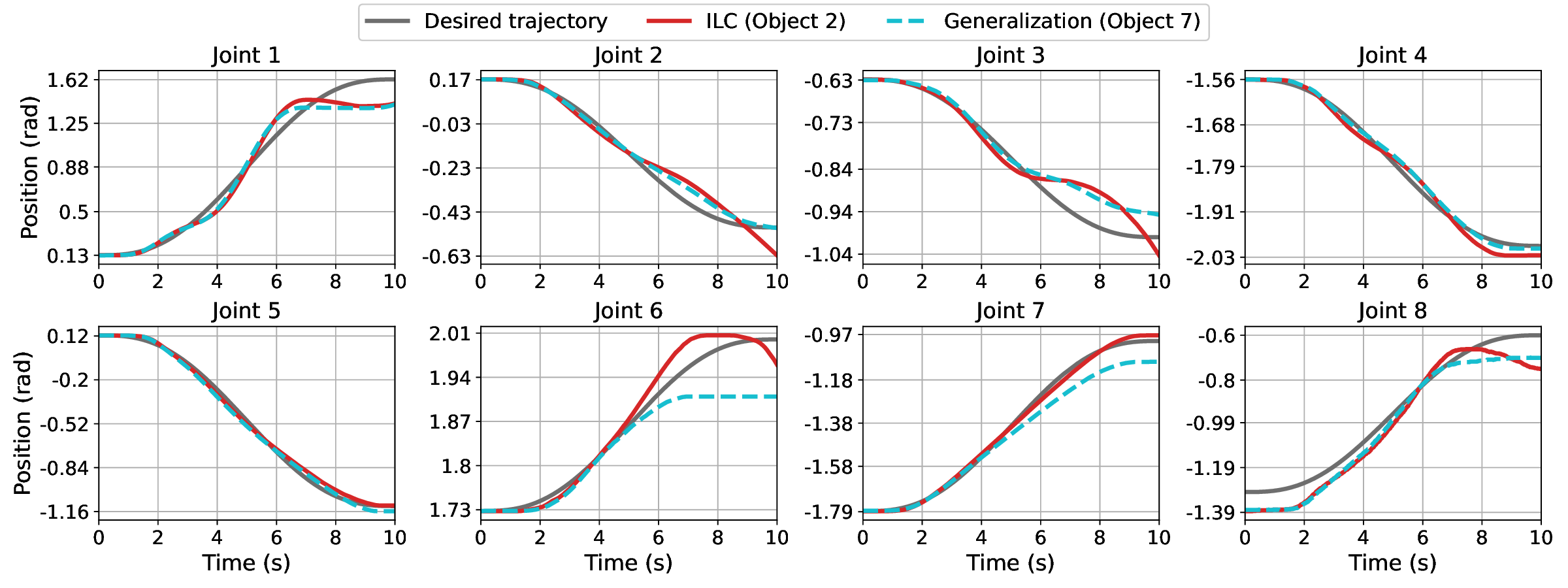}
    \caption{Experiment on Panda manipulator with elastic wrist: comparison between desired trajectory and actual trajectories obtained with ILC and generalized torques.}
    \label{fig:ilchyb_meas}
\end{figure*}

As a final validation, the generalization method is tested on a system consisting of a Panda robot equipped with a variable-stiffness actuator acting as a soft wrist.
The setup is shown in Fig.~\ref{fig:softwrist}.
\revii{This is the general case of the model in Section~\ref{subsec:system_model}, in which rigid and elastic joints coexist: the seven joints of the Panda are rigid and the wrist is elastic, i.e., $n = 8$, $m = 1$, and $\mathcal{E} = \{8\}$, hence $S_\theta = e_8$ in~\eqref{eq:stheta} and $S_\tau = \operatorname{diag}(0, \dots, 0, 1)$ in~\eqref{eq:stau}, while $J$, $D$, and $K$ have rank one.}
The experimental procedure is implemented in ROS, with the Thunder Dynamics tool~\citep{thunderdyn_lib_2025} providing a detailed model of the complex system.
In what follows, the same values of the ILC gains for the Panda robot and the same example and test objects are used (Section~\ref{subsec:exprigid}).
For the soft actuator, the following ILC gains are chosen: $k_{\textnormal{p}_\textnormal{8}} = 0.25$,  $k_{\textnormal{v}_\textnormal{8}} = 0.01$, and $k_{\textnormal{a}_\textnormal{8}} = 0.001$.
The initial and final positions in the minimum jerk trajectory are 
 $q_\textnormal{0}=[0.127, 0.175,  -0.625, -1.555, 0.129,  1.725, -1.798,$
 $-1.300]^\top$ $\si{rad}$ and $q_{\textnormal{f}}=[1.622, -0.547, -1.005,  -2.063,$
 $-1.120, 2.063, -1.003, -0.60]^\top$ $\si{rad}$.
The qbmove Advanced VSA actuator is controlled using a fixed stiffness value, $K = 2.3 \ \si{\newton\meter\per\radian}$, allowing it to work as a SEA actuator.
The mean angle of the motors is then derived as $\theta_\textnormal{m}(t) = q(t) + \frac{1}{K} \tau(t)$, where $q$ is the actual joint position and $\tau$ is the torque computed using ILC.
The learning process comprises $100$ iterations.
Once again, RMSE is employed to evaluate the efficiency of the ILC controller.
Fig.~\ref{fig:ilchyb_RMSEoveriter} illustrates the evolution of the RMSE during the ILC learning phase for all example objects.
% The results of the generalized torques for the test objects compared to the example objects are shown in Fig.~\ref{fig:ilchyb_rmseoverobj}.
Fig.~\ref{fig:ilchyb_rmseoverobj} compares the generalized torques on the test objects with the ILC results on the example objects.

 \begin{figure}%[hb!]
    \centering
    \includegraphics[width=0.85\linewidth,trim={0 0.3cm 0 0},clip]{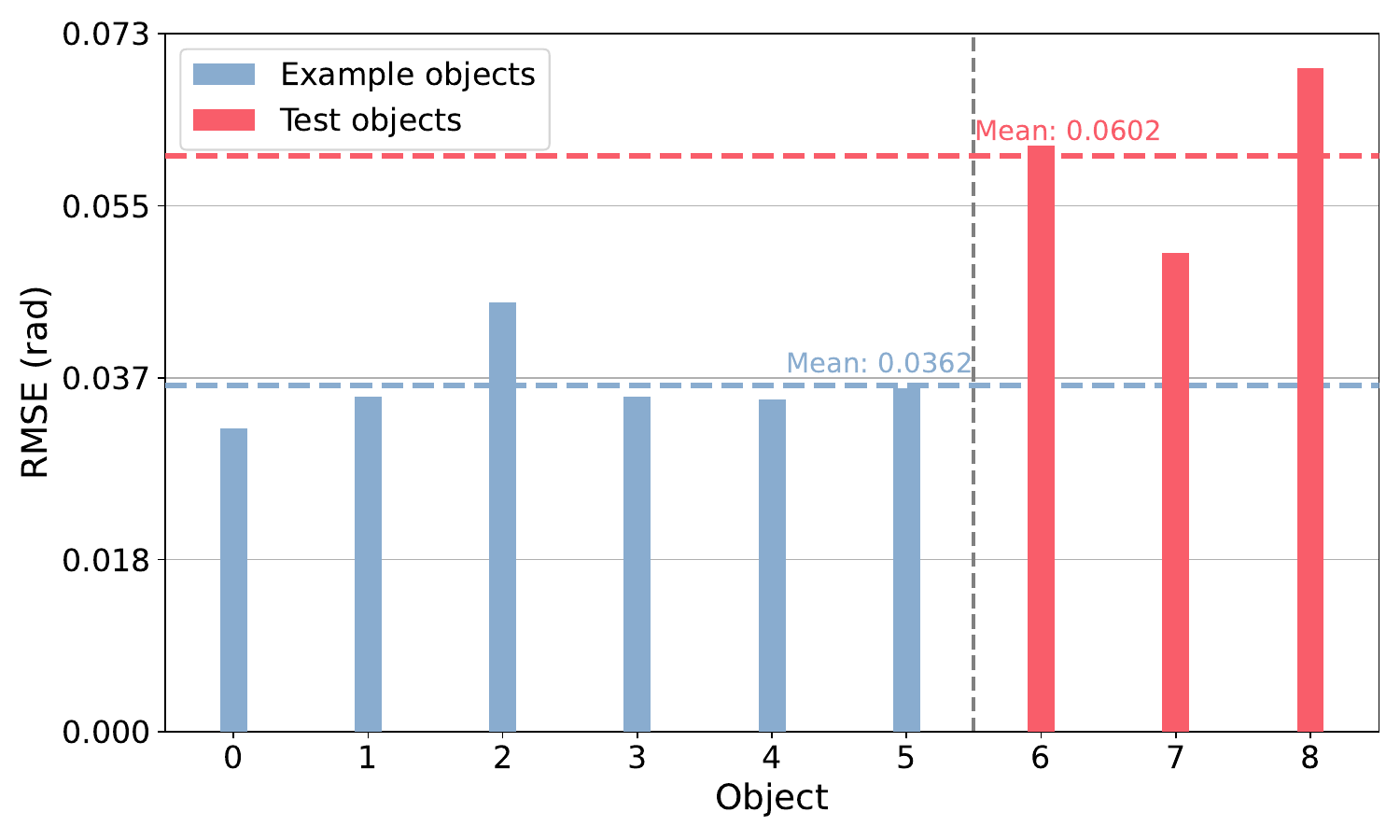}
    \caption{Experiment on Panda manipulator with elastic wrist: comparison of ILC-based and generalized torque performance.}
    \label{fig:ilchyb_rmseoverobj}
\end{figure}
\begin{figure}%[hb!]
    \centering
    \includegraphics[width=0.9\linewidth]{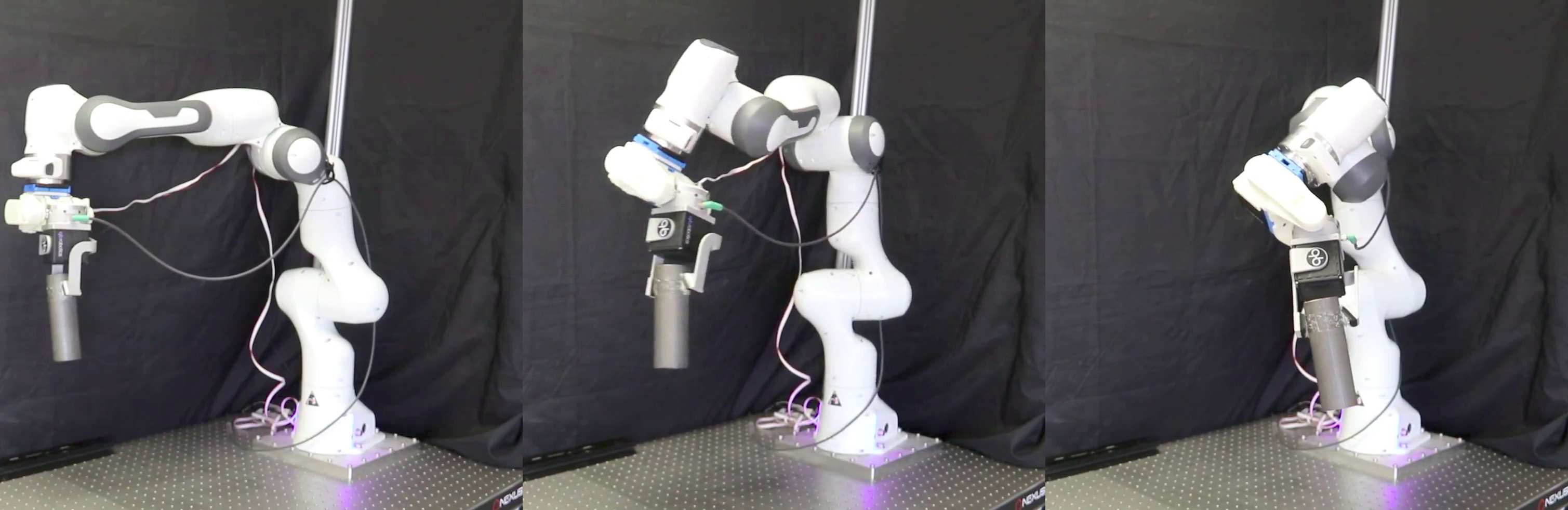}
    \caption{Photo sequence of the Panda robot with an elastic wrist tracking the desired trajectory while carrying Object 8, using the control action computed with the generalized torque method.}
    \label{fig:sequence}
\end{figure}

The tracking performance is highlighted in Fig.~\ref{fig:ilchyb_meas}, where the desired trajectory is compared with the actual trajectories obtained using both ILC and generalized torques across all system joints.
Finally, Fig.~\ref{fig:sequence} presents a sequence of images showing the robot tracking the desired trajectory with test Object 8.

\paragraph*{Discussion.}
In this final experiment, the error trend observed is similar to that obtained for the purely rigid system.
With respect to the ILC controller, Fig.~\ref{fig:ilchyb_RMSEoveriter} shows that the RMSE decreases over the iterations, although not monotonically, in contrast to the previous experiments.
Nevertheless, the controller achieves good trajectory tracking performance for all example objects after the learning phase.
\rev{Consistently} with the rigid system, the global RMSE computed for the example and test objects, as shown in Fig.~\ref{fig:ilchyb_rmseoverobj}, is higher for the generalized torques than for the learned torques\rev{, the mean rising from 0.0362 on the example objects to 0.0602 on the test objects}.
Notably, the difference between the average errors of the ILC and generalized torques is higher than in the rigid case.
\rev{The dynamics of this system is more complex than that of the rigid one, so that the ILC itself achieves a lower tracking accuracy, as reflected by its higher error on the example objects, i.e., $0.0362$ against the $0.0128$ of the rigid case. 
Being built on these example torques, the generalized torques inherit this larger reference error, which propagates as described by the sensitivity analysis in~\eqref{eq:sensanalysis} and yields the higher test error. 
The tracking error nevertheless remains bounded: even in this most demanding scenario, the mean RMSE does not exceed \SI{0.0602}{rad}, and, as confirmed by the joint trajectories in Fig.~\ref{fig:ilchyb_meas} 
and by the photo sequence in Fig.~\ref{fig:sequence}, the desired trajectory is followed for all joints and test objects, so that the generalization remains effective even under joint elasticity.}
%%%%%%%%%%%%%%%%%%%%%%%%%%%%%%%%%%%%%%%%%%%%%%%%%%%%%%%%%%%%%%%%%%%%%%%%%%%
%% Conclusion
%%%%%%%%%%%%%%%%%%%%%%%%%%%%%%%%%%%%%%%%%%%%%%%%%%%%%%%%%%%%%%%%%%%%%%%%%%%
\section{Conclusion}
\label{sec:conclusion}
This paper proposed a method for the generalization of feedforward control inputs, acquired with any controller, to track a desired trajectory with different payloads.
The method was designed for systems with rigid and elastic joints.
The algorithm is independent of the robot's model since it requires only the knowledge of the payload's dynamic parameters to generalize torques without needing to repeat the learning process.
The effectiveness of the method was validated through simulations and experiments across three systems: a 3-DoF soft-articulated manipulator, a rigid Franka Emika Panda, and a soft-rigid system (Panda with an elastic wrist). Simulations covered the first two, while physical experiments were conducted on the Panda-based configurations.
In all setups, the desired torques for the example set were derived using ILC, except for the simulation of the 3-DoF soft-articulated manipulator, where a computed torque controller was employed.
Then, the generalized torques were tested on new objects, demonstrating notable performance.
While the proposed method is promising, there are some limitations, such as handling changes in trajectory, for which the procedure has to be repeated.
Indeed, future work could extend the method to trajectories with different endpoints and to payloads whose dynamic parameters are not known accurately, such as objects with unknown or non-uniform mass distribution.

\printcredits

\section*{Declaration of competing interest}
The authors declare that they have no known competing financial interests or personal relationships that could have appeared to influence the work reported in this paper.

\section*{Supplementary material}
Supplementary material associated with this article can be found in the online version at \url{https://doi.org/10.1016/j.conengprac.2026.107248}.

%%%%%%%%%%%%%%%%%%%%%%%%%%%%%%%%%%%%%%%%%%%%%%%%%%%%%%%%%%%%%%%%%%%%%%%%%%%
%% Bibliography
%%%%%%%%%%%%%%%%%%%%%%%%%%%%%%%%%%%%%%%%%%%%%%%%%%%%%%%%%%%%%%%%%%%%%%%%%%%
\bibliographystyle{cas-model2-names}
\bibliography{Bibliography.bib}

@article{liu2024application,
  title={Application, development and future opportunities of collaborative robots (cobots) in manufacturing: A literature review},
  author={Liu, Li and Guo, Fu and Zou, Zishuai and Duffy, Vincent G},
  journal={International Journal of Human--Computer Interaction},
  volume={40},
  number={4},
  pages={915--932},
  year={2024},
  publisher={Taylor \& Francis},
  doi={10.1080/10447318.2022.2041907}
}

@article{Arimoto84,
author = {Arimoto, Suguru and Kawamura, Sadao and Miyazaki, Fumio},
title = {Bettering operation of Robots by learning},
journal = {Journal of Robotic Systems},
volume = {1},
number = {2},
pages = {123-140},
year = {1984},
doi = {10.1002/rob.4620010203}
}

@ARTICLE{Bristow_2006,
  author={Bristow, D.A. and Tharayil, M. and Alleyne, A.G.},
  journal={IEEE Control Systems Magazine}, 
  title={A survey of iterative learning control}, 
  year={2006},
  volume={26},
  number={3},
  pages={96-114},
  doi={10.1109/MCS.2006.1636313}
}

@ARTICLE{Tayebi_2008,
  author={Tayebi, A. and Abdul, S. and Zaremba, M. B. and Ye, Y.},
  journal={IEEE/ASME Transactions on Mechatronics}, 
  title={Robust Iterative Learning Control Design: Application to a Robot Manipulator}, 
  year={2008},
  volume={13},
  number={5},
  pages={608-613},
  doi={10.1109/TMECH.2008.2004627}
}

@ARTICLE{Saab_2022,
  author={Saab, Samer Said and Shen, Dong and Orabi, Mohamad and Kors, David and Jaafar, Rayana H.},
  journal={IEEE Transactions on Industrial Electronics}, 
  title={Iterative Learning Control: Practical Implementation and Automation}, 
  year={2022},
  volume={69},
  number={2},
  pages={1858-1866},
  doi={10.1109/TIE.2021.3063866}
}

@inproceedings{de2000feedforward,
  title={Feedforward/feedback laws for the control of flexible robots},
  author={De Luca, Alessandro},
  booktitle={Proceedings 2000 ICRA. Millennium Conference. IEEE International Conference on Robotics and Automation. Symposia Proceedings (Cat. No. 00CH37065)},
  volume={1},
  pages={233--240},
  year={2000},
  organization={IEEE},
  doi={10.1109/ROBOT.2000.844064}
}

@article{chhatoi2023optimal,
  title={Optimal control for articulated soft robots},
  author={Chhatoi, Saroj Prasad and Pierallini, Michele and Angelini, Franco and Mastalli, Carlos and Garabini, Manolo},
  journal={IEEE Transactions on Robotics},
  volume={39},
  number={5},
  pages={3671--3685},
  year={2023},
  publisher={IEEE},
  doi={10.1109/TRO.2023.3288837}
}

@article{della2020soft,
  title={Soft robots},
  author={Della Santina, Cosimo and Catalano, Manuel G and Bicchi, Antonio and Ang, M and Khatib, O and Siciliano, B},
  journal={Encyclopedia of Robotics},
  volume={489},
  year={2020},
  publisher={Springer}
}

@article{albu2008soft,
  title={Soft robotics},
  author={Albu-Schaffer, Alin and Eiberger, Oliver and Grebenstein, Markus and Haddadin, Sami and Ott, Christian and Wimbock, Thomas and Wolf, Sebastian and Hirzinger, Gerd},
  journal={IEEE Robotics \& Automation Magazine},
  volume={15},
  number={3},
  pages={20--30},
  year={2008},
  publisher={IEEE},
  doi={10.1109/MRA.2008.927979}
}

@article{liu2019industrial,
  title={Industrial feedforward control technology: a review},
  author={Liu, Lu and Tian, Siyuan and Xue, Dingyu and Zhang, Tao and Chen, YangQuan},
  journal={Journal of Intelligent Manufacturing},
  volume={30},
  number={8},
  pages={2819--2833},
  year={2019},
  publisher={Springer},
  doi={10.1007/s10845-018-1399-6}
}

@article{li2022five,
  title={Five-axis contour error control based on spatial iterative learning},
  author={Li, Jiangang and You, Zhiyang and Li, Yanan and Miao, Enming and Yue, Ruijie},
  journal={IEEE Transactions on Automation Science and Engineering},
  volume={20},
  number={1},
  pages={112--123},
  year={2022},
  publisher={IEEE},
  doi={10.1109/TASE.2022.3142565}
}

@article{chen2024data,
  title={Data-driven methods applied to soft robot modeling and control: A review},
  author={Chen, Zixi and Renda, Federico and Le Gall, Alexia and Mocellin, Lorenzo and Bernabei, Matteo and Dangel, Th{\'e}o and Ciuti, Gastone and Cianchetti, Matteo and Stefanini, Cesare},
  journal={IEEE Transactions on Automation Science and Engineering},
  volume={22},
  pages={2241--2256},
  year={2025},
  publisher={IEEE},
  doi={10.1109/TASE.2024.3377291}
}

@article{pagliarani2024variable,
  title={Variable stiffness structure inspired by seashells},
  author={Pagliarani, Niccolo and Arleo, Luca and De Luca, Giovanna and Pozzi, Jasmine and Cianchetti, Matteo},
  journal={Smart Materials and Structures},
  volume={33},
  number={2},
  pages={025004},
  year={2024},
  publisher={IOP Publishing},
  doi={10.1088/1361-665X/ad1b27}
}

@article{pierallini2022swing,
  title={Swing-up of underactuated compliant arms via iterative learning control},
  author={Pierallini, Michele and Angelini, Franco and Bicchi, Antonio and Garabini, Manolo},
  journal={IEEE Robotics and Automation Letters},
  volume={7},
  number={2},
  pages={3186--3193},
  year={2022},
  publisher={IEEE},
  doi={10.1109/LRA.2022.3144786}
}

@article{pierallini2023iterative,
  title={Iterative learning control for compliant underactuated arms},
  author={Pierallini, Michele and Angelini, Franco and Mengacci, Riccardo and Palleschi, Alessandro and Bicchi, Antonio and Garabini, Manolo},
  journal={IEEE Transactions on Systems, Man, and Cybernetics: Systems},
  volume={53},
  number={6},
  pages={3810--3822},
  year={2023},
  publisher={IEEE},
  doi={10.1109/TSMC.2023.3234403}
}

@article{Mengacci_2021,
author = {Riccardo Mengacci and Franco Angelini and Manuel G Catalano and Giorgio Grioli and Antonio Bicchi and Manolo Garabini},
title ={On the motion/stiffness decoupling property of articulated soft robots with application to model-free torque iterative learning control},
journal = {The International Journal of Robotics Research},
volume = {40},
number = {1},
pages = {348-374},
year = {2021},
doi = {10.1177/0278364920943275}
}

@article{mengacci2021overcoming,
  title={Overcoming the torque/stiffness range tradeoff in antagonistic variable stiffness actuators},
  author={Mengacci, Riccardo and Garabini, Manolo and Grioli, Giorgio and Catalano, Manuel G and Bicchi, Antonio},
  journal={IEEE/ASME Transactions on Mechatronics},
  volume={26},
  number={6},
  pages={3186--3197},
  year={2021},
  publisher={IEEE},
  doi={10.1109/TMECH.2021.3055364}
}

@ARTICLE{Angelini_2020,
  author={Angelini, Franco and Mengacci, Riccardo and Santina, Cosimo Della and Catalano, Manuel G. and Garabini, Manolo and Bicchi, Antonio and Grioli, Giorgio},
  journal={IEEE Robotics and Automation Letters}, 
  title={Time Generalization of Trajectories Learned on Articulated Soft Robots}, 
  year={2020},
  volume={5},
  number={2},
  pages={3493-3500},
  doi={10.1109/LRA.2020.2977268}
}

@ARTICLE{Angelini_2020_trajgen,
AUTHOR={Angelini, Franco  and Della Santina, Cosimo  and Garabini, Manolo  and Bianchi, Matteo  and Bicchi, Antonio },
TITLE={Control Architecture for Human-Like Motion With Applications to Articulated Soft Robots},
JOURNAL={Frontiers in Robotics and AI},
VOLUME={7},
YEAR={2020},
ISSN={2296-9144},
ARTICLE-NUMBER={117},
DOI={10.3389/frobt.2020.00117},
}

@ARTICLE{Angelini_2020_softhandler,
  author={Angelini, Franco and Petrocelli, Cristiano and Catalano, Manuel G. and Garabini, Manolo and Grioli, Giorgio and Bicchi, Antonio},
  journal={IEEE Robotics \& Automation Magazine}, 
  title={SoftHandler: An Integrated Soft Robotic System for Handling Heterogeneous Objects}, 
  year={2020},
  volume={27},
  number={3},
  pages={55-72},
  doi={10.1109/MRA.2019.2955952}
}

@Inbook{Siciliano,
author="Siciliano, Bruno and Sciavicco, Lorenzo and Villani, Luigi and Oriolo, Giuseppe",
title="Dynamics",
bookTitle="Robotics: Modelling, Planning and Control",
year="2009",
publisher="Springer London",
address="London",
pages="247--302",
isbn="978-1-84628-642-1",
doi="10.1007/978-1-84628-642-1_7",
url="https://doi.org/10.1007/978-1-84628-642-1_7"
}

@INPROCEEDINGS{DeLuca2002,
  author={De Luca, A. and Farina, R.},
  booktitle={Proceedings 2002 IEEE International Conference on Robotics and Automation (Cat. No.02CH37292)}, 
  title={Dynamic scaling of trajectories for robots with elastic joints}, 
  year={2002},
  volume={3},
  number={},
  pages={2436-2442 vol.3},
  doi={10.1109/ROBOT.2002.1013597}}

@article{chen2020generalized,
  title={Generalized iterative learning control with mixed system constraints: A gantry robot based verification},
  author={Chen, Yiyang and Chu, Bing and Freeman, Christopher T and Liu, Yanhong},
  journal={Control Engineering Practice},
  volume={95},
  pages={104260},
  year={2020},
  publisher={Elsevier},
  doi={10.1016/j.conengprac.2019.104260}
}

@article{freeman2010iterative,
  title={Iterative learning control for multiple point-to-point tracking application},
  author={Freeman, Chris T and Cai, Zhonglun and Rogers, Eric and Lewin, Paul L},
  journal={IEEE Transactions on Control Systems Technology},
  volume={19},
  number={3},
  pages={590--600},
  year={2010},
  publisher={IEEE},
  doi={10.1109/TCST.2010.2051670}
}

@article{pierallini2025fishing,
  title={Fishing for Data: Modeling, Optimal Planning, and Iterative Learning Control for Flexible Link Robots},
  author={Pierallini, Michele and Gopanunni, Ramesh Krishnan Muttathil and Angelini, Franco and Bicchi, Antonio and Garabini, Manolo},
  journal={IEEE Transactions on Control Systems Technology},
  year={2025},
  publisher={IEEE},
  doi={10.1109/TCST.2025.3550030}
}

@article{palleschi2021fast,
  title={Fast and safe trajectory planning: Solving the cobot performance/safety trade-off in human-robot shared environments},
  author={Palleschi, Alessandro and Hamad, Mazin and Abdolshah, Saeed and Garabini, Manolo and Haddadin, Sami and Pallottino, Lucia},
  journal={IEEE Robotics and Automation Letters},
  volume={6},
  number={3},
  pages={5445--5452},
  year={2021},
  publisher={IEEE},
  doi={10.1109/LRA.2021.3076968}
}

@article{zeng2020tossingbot,
  title={Tossingbot: Learning to throw arbitrary objects with residual physics},
  author={Zeng, Andy and Song, Shuran and Lee, Johnny and Rodriguez, Alberto and Funkhouser, Thomas},
  journal={IEEE Transactions on Robotics},
  volume={36},
  number={4},
  pages={1307--1319},
  year={2020},
  publisher={IEEE},
  doi={10.1109/TRO.2020.2988642}
}

@article{gabellieri2022autonomous, 
  title={Autonomous unwrapping of general pallets: a novel robot for logistics exploiting contact-based planning},
  author={Gabellieri, Chiara and Palleschi, Alessandro and Pallottino, Lucia and Garabini, Manolo},
  journal={IEEE Transactions on Automation Science and Engineering},
  volume={20},
  number={2},
  pages={1194--1211},
  year={2022},
  publisher={IEEE},
  doi={10.1109/TASE.2022.3182362}
}

@article{surati2021pick,
  title={Pick and place robotic arm: a review paper},
  author={Surati, Sharath and Hedaoo, Shaunak and Rotti, Tushar and Ahuja, Vaibhav and Patel, Nishigandha},
  journal={Int. Res. J. Eng. Technol},
  volume={8},
  number={2},
  pages={2121--2129},
  year={2021}
}

@article{lobbezoo2021reinforcement,
  title={Reinforcement learning for pick and place operations in robotics: A survey},
  author={Lobbezoo, Andrew and Qian, Yanjun and Kwon, Hyock-Ju},
  journal={Robotics},
  volume={10},
  number={3},
  pages={105},
  year={2021},
  publisher={MDPI},
  doi={10.3390/robotics10030105}
}

@Inbook{Albu-Schäffer2016,
author="Albu-Sch{\"a}ffer, Alin
and Bicchi, Antonio",
editor="Siciliano, Bruno
and Khatib, Oussama",
title="Actuators for Soft Robotics",
bookTitle="Springer Handbook of Robotics",
year="2016",
publisher="Springer International Publishing",
address="Cham",
pages="499--530",
isbn="978-3-319-32552-1",
doi="10.1007/978-3-319-32552-1_21",
url="https://doi.org/10.1007/978-3-319-32552-1_21"
}

@article{Spong_1987,
    author = {Spong, M. W.},
    title = {Modeling and Control of Elastic Joint Robots},
    journal = {Journal of Dynamic Systems, Measurement, and Control},
    volume = {109},
    number = {4},
    pages = {310-318},
    year = {1987},
    month = {12},
    issn = {0022-0434},
    doi = {10.1115/1.3143860},
    url = {https://doi.org/10.1115/1.3143860},
}

@ARTICLE{adaptiveCT_simonini_2025,
  author={Simonini, Giorgio and Baracca, Marco and Cavaliere, Tommaso V. and Bicchi, Antonio and Salaris, Paolo},
  journal={IEEE Access}, 
  title={A Novel Formulation for Adaptive Computed Torque Control Enabling Low Feedback Gains in Highly Dynamical Tasks}, 
  year={2025},
  volume={13},
  number={},
  pages={69898-69909},
  doi={10.1109/ACCESS.2025.3561635}
}

@INPROCEEDINGS{middletone_adaptivect_1986,
  author={Middletone, R. H. and Goodwin, G. C.},
  booktitle={1986 25th IEEE Conference on Decision and Control}, 
  title={Adaptive computed torque control for rigid link manipulators}, 
  year={1986},
  volume={},
  number={},
  pages={68-73},
  doi={10.1109/CDC.1986.267156}}

@Article{thunderdyn_lib_2025,
AUTHOR = {Baracca, Marco and Simonini, Giorgio and Tolomei, Simone and De Santis, Yuri and Rosa Brusin, Paolo and Angeli, Stefano and Gabiccini, Marco and Bicchi, Antonio and Salaris, Paolo},
TITLE = {Thunder Dynamics: A C++ Tool for Adaptive Control of Serial Manipulators},
JOURNAL = {Robotics},
VOLUME = {14},
YEAR = {2025},
NUMBER = {9},
ARTICLE-NUMBER = {126},
ISSN = {2218-6581},
DOI = {10.3390/robotics14090126}
}

@book{golub1996van,
  author    = {Golub, Gene H. and Van Loan, Charles F.},
  title     = {Matrix Computations},
  edition   = {4th},
  year      = {2013},
  publisher = {Johns Hopkins University Press},
  address   = {Baltimore, MD, USA},
  isbn      = {978-1-4214-0794-4}
}

@incollection{Hollerbach2016,
  author    = {Hollerbach, John and Khalil, Wisama and Gautier, Maxime},
  title     = {Model Identification},
  booktitle = {Springer Handbook of Robotics},
  edition   = {2nd},
  editor    = {Siciliano, Bruno and Khatib, Oussama},
  chapter   = {6},
  pages     = {113--138},
  publisher = {Springer International Publishing},
  address   = {Cham},
  year      = {2016},
  doi       = {10.1007/978-3-319-32552-1_6}
}

@article{Atkeson1986,
  author  = {Atkeson, Christopher G. and An, Chae H. and Hollerbach, John M.},
  title   = {Estimation of Inertial Parameters of Manipulator Loads and Links},
  journal = {The International Journal of Robotics Research},
  volume  = {5},
  number  = {3},
  pages   = {101--119},
  year    = {1986},
  doi     = {10.1177/027836498600500306}
}

@article{SlotineLi1987,
  author  = {Slotine, Jean-Jacques E. and Li, Weiping},
  title   = {On the Adaptive Control of Robot Manipulators},
  journal = {The International Journal of Robotics Research},
  volume  = {6},
  number  = {3},
  pages   = {49--59},
  year    = {1987},
  doi     = {10.1177/027836498700600303}
}

@article{Arteaga2024,
  author  = {Arteaga, Marco A.},
  title   = {On the Exact Parameter Estimation for Robot Manipulators Without Persistence of Excitation},
  journal = {IEEE Transactions on Automatic Control},
  volume  = {69},
  number  = {1},
  pages   = {410--417},
  year    = {2024},
  doi     = {10.1109/TAC.2023.3269249}
}

@Inbook{Roman1992,
author="Roman, Steven",
title="Affine Geometry",
bookTitle="Advanced Linear Algebra",
year="1992",
publisher="Springer New York",
address="New York, NY",
pages="315--328",
isbn="978-1-4757-2178-2",
doi="10.1007/978-1-4757-2178-2_16",
url="https://doi.org/10.1007/978-1-4757-2178-2_16"
}

    %GP
    \bio{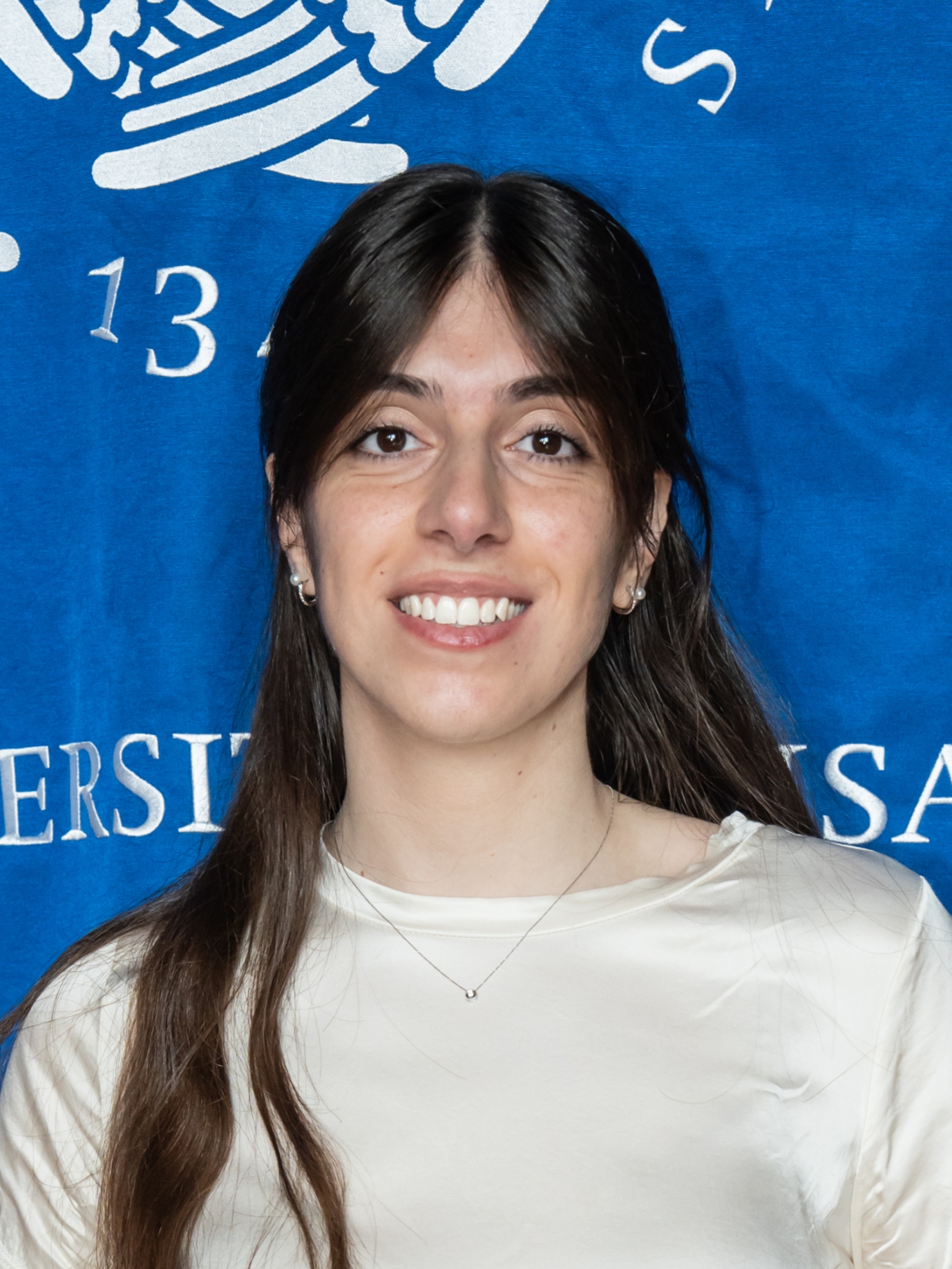}Giorgia Pazzi is a PhD student in Information Engineering (Robotics) at the University of Pisa, Italy. She received her B.Sc. in Computer and Electronic Engineering and her M.Sc. in Robotics and Automation Engineering. Her research focuses on learning-based control and generalization of control actions for feedforward controllers, along with modeling and control of soft robotic systems.
    Her research is conducted on several robotic platforms, including soft articulated manipulators, quadruped robots, and robotic legs for the testing of lower-limb exoskeletons.

    \endbio
    %GS
    \bio{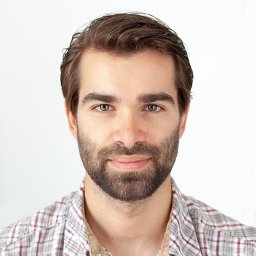}Giorgio Simonini received the bachelor’s degree in electronics, the master’s degree in robotics engineering, and the Ph.D. degree in Robotics and Intelligent Machines from the University of Pisa, as part of the national doctoral program DRIM.
    His research interests include efficient planning and control techniques for both traditional and elastic-joint robots. A significant part of his work involves adaptive control and object manipulation, with a particular focus on object throwing using robotic manipulators.
    \endbio
    %MP
    \bio{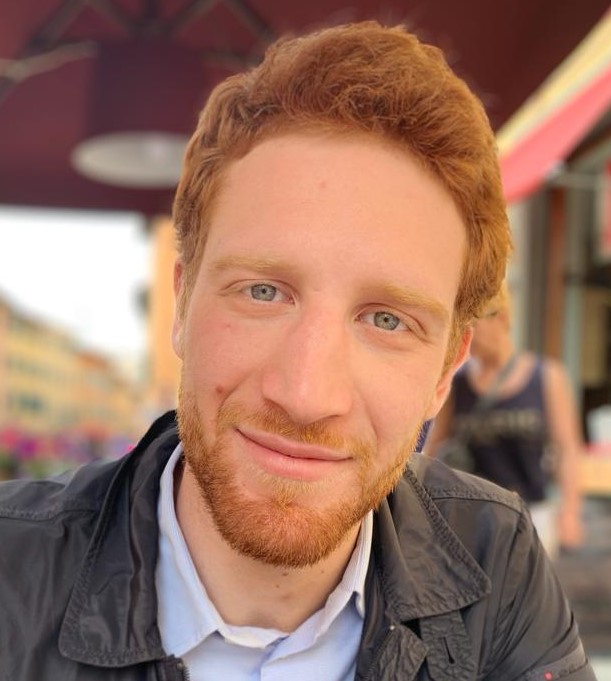}
    Michele Pierallini received a B.S. degree in biomedical engineering in 2017,  M.S. degree (cum laude) in automation and robotics engineering in 2020, and the Ph.D. degree (cum laude) in robotics in 2024 from the University of Pisa, Pisa, Italy, where he is currently working at the Research Center “Enrico Piaggio.” His current research focuses on the control of soft robotic systems and iterative learning control.
    \endbio
    %PS
    \bio{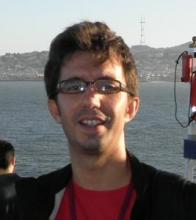}
    PAOLO SALARIS (Member, IEEE) received the M.Sc. degree in electrical engineering and the Ph.D. degree in automation and bioengineering from the University of Pisa, in 2007 and 2011, respectively. He was a Visiting Scholar with the Beckman Institute for Advanced Science and Technology, University of Illinois Urbana–Champaign, in 2009. He was a Postdoctoral Researcher with the Research Center ‘‘E. Piaggio,’’ from March 2011 to January 2014, and LAAS-CNRS, Toulouse, from February 2014 to July 2015. From 2015 to 2019, he was a Permanent Researcher (CRCN) with INRIA Sophia Antipolis Mediterranee, and from 2019 to 2022, he was an Assistant Professor with the University of Pisa, where he became an Associate Professor, in 2022. His research interests include optimal motion planning and control, active sensing, optimal sensor placements, optimal estimation, multi-robot, and nonholonomic systems.
    \endbio
    %GG
    \bio{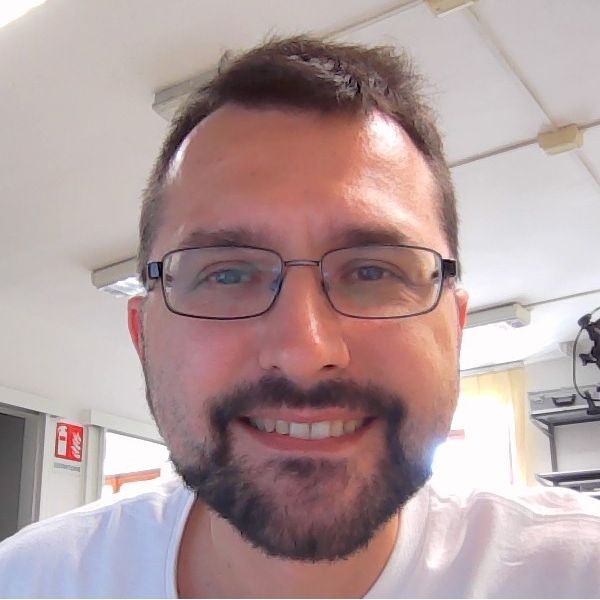}
    Giorgio Grioli received the Ph.D. degree in Robotics, Automation and Bioengineering in 2011 from the University of Pisa, Italy. He is currently a Senior Researcher with the Department of Information Engineering and the E. Piaggio Research Center, University of Pisa, and Visiting Scientist at Italian Institute of Technology. His research focuses on soft robotics, robotic and prosthetic hands, variable impedance actuation, haptics, and physical human–robot interaction. He has contributed to the development of several widely adopted robotic and prosthetic systems, authored more than 150 scientific publications, holds multiple international patents, and co-founded two robotics spin-off companies qbrobotics and Exsensia. 
    \endbio
    %MG
    \bio{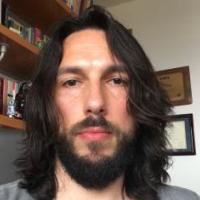}
    Manolo Garabini graduated in Mechanical Engineering and received the Ph.D. in Robotics from the University of Pisa where he is currently employed as Associate Professor. He contributed to the realization of modular Variable Stiffness Actuators. He coordinated the H2020 EU Research Project Natural Intelligence. His main research interests include the design, planning, and control of soft adaptive robots. 
    \endbio
    %FA
    \bio{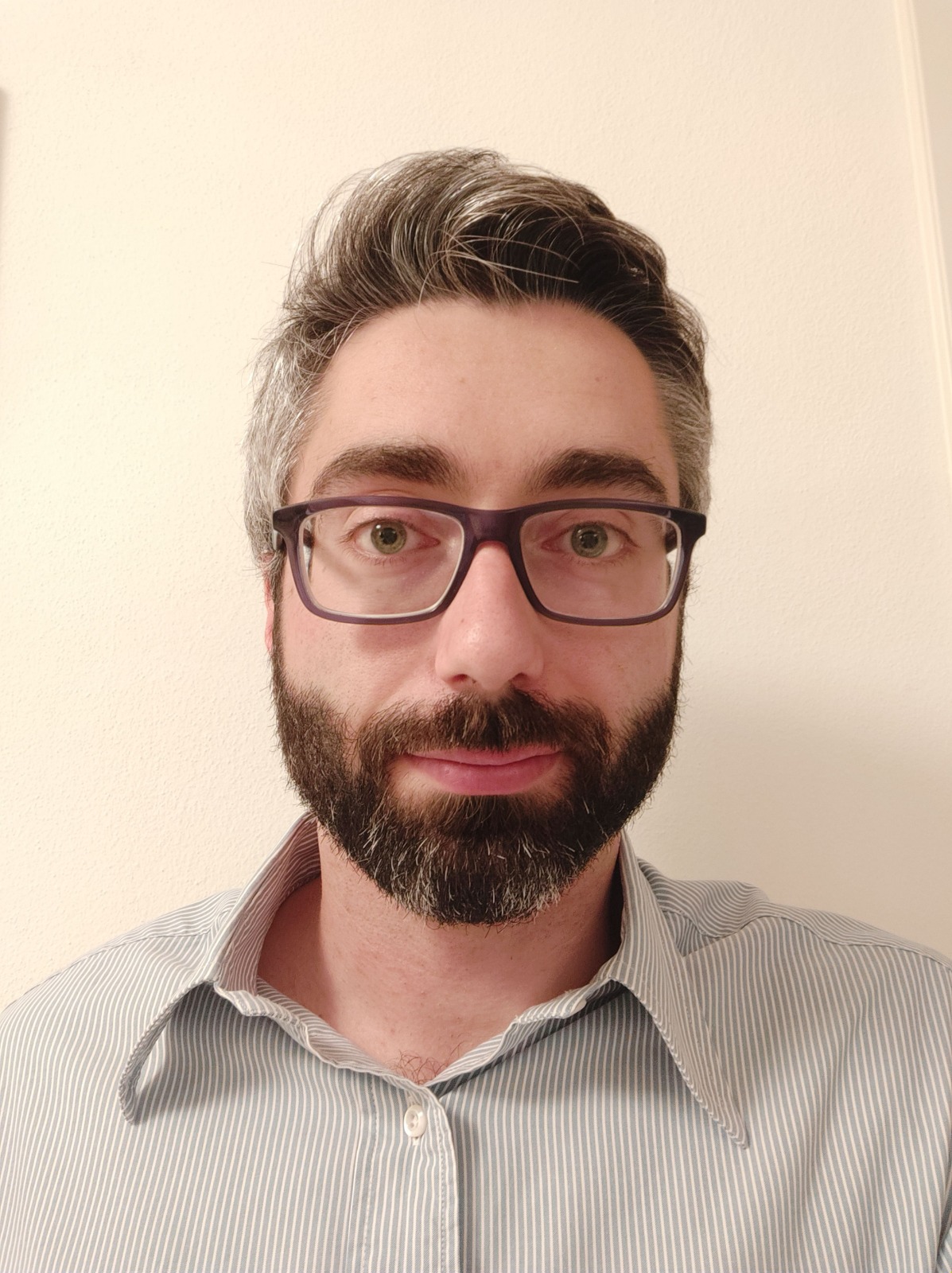}
    Franco Angelini received the B.S. degree in computer engineering in 2013, the M.S. degree in automation and robotics engineering in 2016, and the Ph.D. degree in robotics in 2020 from the University of Pisa, Pisa, Italy. Franco is currently an Assistant Professor at the University of Pisa. His main research interests are control of soft robotic systems, impedance planning, and robotic environmental monitoring.
    \endbio

\end{document}